%% file: arxiv_full_version.tex
\documentclass{article}

\usepackage[utf8]{inputenc} 
\usepackage[T1]{fontenc}    
\usepackage{hyperref}       
\usepackage{url}            
\usepackage{booktabs}       
\usepackage{amsfonts}       
\usepackage{nicefrac}       
\usepackage{microtype}      
\usepackage{fullpage}
\usepackage{breakcites}

\usepackage[pdftex]{graphicx} 		
\usepackage{color, soul} 	

\usepackage{algorithm}		
\usepackage{algorithmic}

\usepackage{amsmath,amsthm,amssymb}
\usepackage{thmtools,thm-restate}

\usepackage{cleveref}

\newcommand{\R}{\mathbb{R}}
\newcommand{\E}{\mathbb{E}}

\renewcommand{\log}{\lg}
\DeclareMathOperator{\sign}{sign}

\newcommand{\margin}{\mathrm{margin}}

\theoremstyle{definition}
\newtheorem{definition}{Definition}

\theoremstyle{plain}

\newtheorem{corollary}{Corollary}[section]
\newtheorem{lemma}{Lemma}[section]
\newtheorem{theorem}{Theorem}[section]

\newtheorem{innercustomlem}{Restatement of Lemma}
\newenvironment{customlem}[1]
  {\renewcommand\theinnercustomlem{#1}\innercustomlem}
  {\endinnercustomlem}

\newtheorem{innercustomthm}{Restatement of Theorem}
\newenvironment{customthm}[1]
  {\renewcommand\theinnercustomthm{#1}\innercustomthm}
  {\endinnercustomthm}

\title{Optimal Minimal Margin Maximization with Boosting}

\author{
   Allan Gr\o nlund\footnote{All authors contributed evenly.}\ \thanks{Aarhus University.
   Email: \texttt{jallan@cs.au.dk}.  }  
   \and
   Kasper Green Larsen\thanks{Aarhus University.
     Email: \texttt{larsen@cs.au.dk}. Supported by a Villum
     Young Investigator Grant and an AUFF Starting Grant.
   }
   \and
   Alexander Mathiasen\thanks{Aarhus University.
     Email: \texttt{alexander.mathiasen@gmail.com}. Supported by
     an AUFF Starting Grant.
   }
}

\begin{document}

\date{}

\maketitle

\begin{abstract}
Boosting algorithms produce a classifier by iteratively combining base
hypotheses. It has been observed experimentally that the generalization error
keeps improving even after achieving zero training error. One popular
explanation attributes this to improvements in margins. A common goal
in a long line of research, is to maximize the smallest margin using as few
base hypotheses as
possible, culminating with the AdaBoostV algorithm by \cite{ratsch2005efficient}. 
The AdaBoostV algorithm was later conjectured to
yield an optimal trade-off between number of hypotheses trained and the minimal
margin over all training points \cite{nie2013open}. 
Our main contribution is a new algorithm refuting this conjecture. 
Furthermore, we prove a lower bound which implies that our new
algorithm is optimal.
\end{abstract}

\input{intro}
\input{upperbound}
\input{inequality}

\input{experiments}

\input{conclusion}

\bibliography{boosting}
\bibliographystyle{apalike}

\appendix
\input{appendix}

\end{document}

%% file: intro.tex
\section{Introduction}
Boosting is one of the most famous and succesful ideas in learning. Boosting algorithms are meta algorithms that produce highly accurate classifiers by combining already existing less accurate classifiers. 
Probably the most famous boosting algorithm is AdaBoost by Freund and Schapire~\cite{freund1995desicion}, who won the 2003 Gödel Prize for their work.


AdaBoost was designed for binary classification and works by combining base hypotheses learned by a given base learning algorithm into a weighted sum that represents the final classifier. This weighed set of base hypotheses is constructed iteratively in rounds, each round constructing a new base hypothesis that focuses on the training data misclassified by the previous base hypotheses constructed. More precisely, AdaBoost takes training data $D=\{(x_i, y_i) \mid x_i\in \mathbb{R}^d, y_i\in \{-1, +1\} \}_{i=1}^n$ and constructs a linear combination classifier $\mathrm{sign}(\sum_{t=1}^T\alpha_t h_t(x))$, where $h_t$ is the base hypothesis learned in the $t$'th iteration and  $\alpha_t$ is the corresponding weight. 


It has been proven that AdaBoost decreases the training error exponentially fast if each base hypothesis is slightly better than random guessing on the weighed data set it is trained on \cite{freund1999short}. Concretely, if $\epsilon_t$ is the error of $h_t$ on the weighed data set used to learn $h_t$ then the linear combination has training error at most $\text{exp}(-2\sum_{t=1}^T{(1/2-\epsilon_t)^2})$. If each $\epsilon_t$ is at most a half minus a fixed constant, then the training error is less than $1/n$ after $O(\log n)$ rounds which means the all training points are classified correctly.
%
%
Quite surprisingly, experiments show that continuing the AdaBoost algorithm even after the training data is perfectly classified, making the model more and more complex, continues to improve generalization \cite{schapire1998boosting}. The most prominent approach to explaining this generalization phenomenon considers \emph{margins} \cite{schapire1998boosting}. The margin of a point $x_i$ is
\begin{align*}
\margin(x_i) = \frac{y_i \sum_{t=1}^T \alpha_t h_t(x_i)}
{\sum_{t=1}^T \lvert \alpha_t \rvert}.
\end{align*}
%
For binary classification, if each $h_t(x)\in[-1, +1]$, then the margin of a point is a number between -1 and +1. Notice that a point has positive margin if it is classified correctly and negative margin if it is classified incorrectly. It has been observed experimentally that the margins of the training points usually increase when training, even after perfectly classifying the training data. 
This has inspired several bounds on generalization error that depend on the distribution of margins \cite{schapire1998boosting, breiman1999prediction, koltchinskii2001some, emargin, kmargin}. The conceptually simplest of these bounds depend only on the minimal margin, which is the margin of the point $x_i$ with minimal margin. The point $x_i$ with minimal margin can be interpreted as the point the classifier struggles the most with. This has inspired a series of algorithms with guarantees on the minimal margin \cite{breiman1999prediction, grove1998boosting, bennett2000column, ratsch2002maximizing, ratsch2005efficient}.

These algorithms have the following goal: Let $\mathcal{H}$ be the (possibly infinite) set of all base hypotheses that may be returned by the base learning algorithm. Suppose the best possible minimal margin on some training data for any linear combination of $h \in \mathcal{H}$ is $\rho^*$, i.e.
\begin{align*}
\rho^* = \max_{\alpha \neq 0}  \left( \min_i \sum_{h\in \mathcal{H}} \frac{y_i \sum_{h \in \mathcal{H}} \alpha_h h(x_i)}
{\sum_{h \in \mathcal{H}} |\alpha_h|} \right) .
\end{align*}

Given some precision parameter $v$, the goal is to construct a linear combination with minimal margin at least $\rho=\rho^*-v$ using as few hypotheses as possible. In this case we say that the linear combination has a gap of $v$.
The current state of the art is AdaBoostV \cite{ratsch2005efficient}. It guarantees a gap of $v$ using $O(\log(n) / v^2)$ hypotheses. It was later conjectured that there exists data sets $D$ and a corresponding set of base hypotheses $\mathcal{H}$, such that any linear combination of base hypotheses from $\mathcal{H}$ must use at least $\Omega(\log(n)/v^2)$ hypotheses to achieve a gap of $v$ for any $\sqrt{\lg n/n} \leq v \leq a_1$ for some constant $a_1 > 0$. This would imply optimality of AdaBoostV. This conjecture was published as an open problem in the Journal of Machine Learning Research \cite{nie2013open}.

Our main contribution is a refutal of this conjecture. We refute the conjecture by introducing a new algorithm called SparsiBoost, which guarantees a gap of $v$ with just $T=O(\log(nv^2)/v^2)$ hypotheses.
When $v \leq n^{o(1)}/\sqrt{n}$, SparsiBoost has $T=O(\log(n^{o(1)})/v^2)=o(\log (n)/v^2)$, which is asymptotically better than AdaBoostV's $T=O(\log(n)/v^2)$ guarantee. Moreover, it also refutes the conjectured lower bound. Our algorithm involves a surprising excursion to the field of combinatorial discrepancy minimization.
We also show that our algorithm is the best possible. That is, there exists data sets $D$ and corresponding set of base hypotheses $\mathcal{H}$, such that any linear combination of base hypotheses from $\mathcal{H}$ with a gap of $v$, must use at least $T=\Omega(\log(nv^2)/v^2)$ hypotheses. 

This work thus provides the final answer to over a decade's research into understanding the trade-off between minimal margin and the number of hypotheses: Given a gap $v$, the optimal number of hypotheses is $T=\Theta(\log(nv^2)/v^2)$ for any $\sqrt{1/n} \leq v \leq a_1$ where $a_1 > 0$ is a constant. Notice that smaller values for $v$ are irrelevant since it is always possible to achieve a gap of zero using $n+1$ base hypotheses. This follows from Carathéodory's Theorem.

\subsection{Previous Work on Minimal Margin}
\paragraph{Upper Bounds.}
\cite{breiman1999prediction} introduced Arc-GV, which was the first algorithm that guaranteed to find a finite number of hypotheses $T<\infty$ with gap zero ($v=0$). As pointed out by \cite{ratsch2005efficient}, one can think of Arc-GV as a subtle variant of AdaBoost where the weights $\alpha_t$ of the hypotheses are slightly changed. If AdaBoost has hypothesis weight $\alpha_t$, then Arc-GV chooses the hypothesis weight $\alpha_t'=\alpha_t + x$ for some $x$ that depends on the minimal margin of $h_1,\dots, h_t$. A few years later, \cite{grove1998boosting} and \cite{bennett2000column} introduced DualLPBoost and LPBoost which both have similar guarantees. 

\cite{ratsch2002maximizing} introduced AdaBoost$_\rho$, which was the first algorithm to give a guarantee on the gap achieved in terms of the number of hypotheses used. Their algorithm takes a parameter $\rho\le \rho^*$ that serves as the target margin one would like to achieve. It then guarantees a minimal margin of $\rho - \mu$ using $T=O(\log(n)/\mu^2)$ hypotheses. One would thus like to choose $\rho=\rho^*$.  If $\rho^*$ is unknown, it can be found up to an additive approximation of $v$ by binary searching using AdaBoost$_\rho$. This requires an additional $O(\log 1/v)$ calls to AdaBoost$_\rho$, resulting in $O(\log(n)/v^2)\log(1/v)$ iterations of training a base hypothesis to find the desired linear combination of $T=O(\lg(n)/v^2)$ base hypotheses. Similar to Arc-GV, AdaBoost$_\rho$ differs from AdaBoost only in choosing the weights $\alpha_t$. 
Instead of having the additional term depend on the minimal margin of $h_1,\dots,h_t$, it depends only on the estimation parameter $\rho$. 

A few years later, \cite{ratsch2005efficient} introduced AdaBoostV. It is a clever extension of AdaBoost$_\rho$ that uses an adaptive estimate of $\rho^*$ to remove the need to binary search for it. It  achieves a gap of $v$ using $T=O(\log(n)/v^2)$ base hypotheses and no extra iterations of training.

\paragraph{Lower Bounds.}
\cite{klein1999number} showed a lower bound for a seemingly unrelated game theoretic problem. It was later pointed out by \cite{nie2013open} that their result implies the following lower bound for boosting: there exists a data set of $n$ points and a corresponding set of base hypotheses $\mathcal{H}$, such that any linear combination of $T\in[\log n; \sqrt{n}]$ base hypotheses must have a gap of $v=\Omega(\sqrt{\log(n)/T})$. Rewriting in terms of $T$ we get $T=\Omega\left( \log(n)/v^2\right)$ for $\sqrt{\lg(n)}/n^{1/4} \leq v \leq a_1$ for some constant $a_1 > 0$.


\cite{nie2013open} conjectured that Klein and Young's lower bound of $v = \Omega(\sqrt{\lg(n)/T})$ holds for all $T\le a_1 \cdot n$ for some constant $a_1 > 0$. Rewriting in terms of $T$, they conjecture that $T=\Omega(\log(n)/v^2)$ holds for $\sqrt{ a_2/n}\le v \le a_3$ where $a_2,a_3 > 0$ are some constants.


\subsection{Our Results On Minimal Margin}
\label{sec:ourmini}
Our main result is a novel algorithm, called SparsiBoost, which refutes the conjectured lower bound in \cite{nie2013open}. Concretely, SparsiBoost guarantees a gap of $v$ with just $T=O(\log(nv^2)/v^2)$ hypotheses.
At a first glance it might seem SparsiBoost violates the lower bound of Klein and Young. Rewriting in terms of $v$, our upper bound becomes $v=O(\sqrt{\log(n/T)/T})$ (see \Cref{sec:appen} for details). When $T\le \sqrt{n}$ (the range of parameters where their lower bound applies), this becomes $v=O(\sqrt{\log(n)/T})$ which does not violate Klein and Young's lower bound. Moreover, our upper bound explains why both \cite{klein1999number} and \cite{nie2013open} were not able to generalize the lower bound to all $T = O(n)$: When $T = n^{1-o(1)}$, our algorithm achieves a gap of $v = O(\sqrt{\lg(n^{o(1)})/T}) = o(\sqrt{\lg(n)/T})$.

The high level idea of SparsiBoost is as follows: Given a desired gap $v$, we use AdaBoostV to find $m=O(\lg(n)/v^2)$ hypotheses $h_1,\dots,h_{m}$ and weights $w_1,\dots,w_{m}$ such that $\sum_i w_i h_i$ achieves a gap of $v/2$. We then carefully ``sparsify'' the vector $w = (w_1,\dots,w_{m})$ to obtain another vector $w'$ that has at most $T=O(\lg(nv^2)/v^2)$ non-zeroes. Our sparsification is done such that the margin of every single data point changes by at most $v/2$ when replacing $\sum_i w_i h_i$ by $\sum_i w'_i h_i$. In particular, this implies that the minimum margin, and hence gap, changes by at most $v/2$. We can now safely ignore all hypotheses $h_i$ where $w'_i=0$ and we have obtained the claimed gap of at most $v/2+v/2 =v$ using $T=O(\lg(nv^2)/v^2)$ hypotheses.

Our algorithm for sparsifying $w$ gives a general method for sparsifying a vector while approximately preserving a matrix-vector product. We believe this result may be of independent interest and describe it in more detail here: The algorithm is given as input a matrix $U\in[-1,+1]^{n \times m}$ and a vector $w \in \R^m$ where $\|w\|_1=1$. It then finds a vector $w'$ such that $\|Uw-Uw'\|_\infty = O(\log(n/T)/T), \|w'\|_0\le T$ and $\|w'\|_1=1$. Here $\|x\|_1=\sum_i|x_i|$, $\|x\|_\infty=\text{max}_i|x_i |$ and $\|x\|_0$ denotes the number of non-zero entries of $x$. When we use this result in SparsiBoost, we will define the matrix $U$ as the ``margin matrix'' that has $u_{ij} = y_ih_j(x_i)$. Then $(Uw)_i = \margin(x_i)$ and the guarantee $\|Uw-Uw'\|_\infty = O(\log(n/T)/T)$ will ensure that the margin of every single point changes by at most $O(\lg(n/T)/T)$ if we replace the weights $w$ by $w'$. Our algorithm for finding $w'$ is based on a novel connection to the celebrated but seemingly unrelated ``six standard devitations suffice'' result by \cite{spencer1985six} from the field of combinatorial discrepancy minimization. 

When used in SparsiBoost, the matrix $U$ is defined from the output of AdaBoostV, but the vector sparsification algorithm could just as well be applied to the hypotheses output by any boosting algorithm. Thus our results give a general method for sparsifying a boosting classifier while approximately preserving the margins of all points.


We complement our new upper bound with a matching lower bound. More concretely, we prove that there exists data sets $D$ of $n$ points and a corresponding set of base hypotheses $\mathcal{H}$, such that any linear combination of $T$ base hypotheses must have a gap of at least $v = \Omega(\sqrt{\lg(n/T)/T})$ for any $\lg n \leq T \leq a_1 n$ where $a_1>0$ is a constant. Rewriting in terms of $T$, one must use $T = \Omega(\lg(nv^2)/v^2)$ hypotheses from $\mathcal{H}$ to achieve a gap of $v$. This holds for any $v$ satisfying $\sqrt{a_2/n} < v \leq a_3$ for constants $a_2,a_3 > 0$ (see \Cref{sec:appen} for details). Interestingly, our lower bound proof also uses the discrepancy minimization upper bound by \cite{spencer1985six} in a highly non-trivial way. Our lower bound also shows that our vector sparsification algorithm is optimal for any $T \leq n/C$ for some universal constant $C > 0$.

\subsection{Doubts on Margin Theory and other Margin Bounds}
The first margin bound on boosted classifiers was introduced by \cite{schapire1998boosting}. Shortly after, \cite{breiman1999prediction} introduced a sharper minimal margin bound alongside Arc-GV. Experimentally Breimann found that Arc-GV produced better margins than AdaBoost on 98 percent of the training data, however, AdaBoost still obtained a better test error. This seemed to contradict margin theory: according to margin theory, better margins should imply better generalization. This caused Breimann to doubt margin theory. It was later discovered by \cite{reyzin2006boosting} that the comparison was unfair due to a difference in the complexity of the base hypotheses used by AdaBoost and Arc-GV. \cite{reyzin2006boosting} performed a variant of Breimann's experiments with decision stumps to control the hypotheses complexity. They found that even though Arc-GV produced a better minimal margin, AdaBoost produced a larger margin on almost all other points \cite{reyzin2006boosting} and that AdaBoost generalized better.

A few years later, \cite{emargin} introduced a sharper margin bound than Breimann's minimal margin bound. The generalization bound depends on a term called the \textit{Equilibrium Margin}, which itself depends on the margin distribution in a highly non-trivial way. This was followed by the $k$-margin bound by \cite{kmargin} that provide generalization bounds based  on the $k$'th smallest margin for any $k$. The bound gets weaker with increasing $k$, but stronger with increasing margin. In essence, this means that we get stronger generalization bounds if the margins are large for small values of $k$.

Recall from the discussion in Section~\ref{sec:ourmini} that our sparsification algorithm preserves all margins to within $O(\sqrt{\lg(n/T)/T})$ additive error. We combined this result with AdaBoostV to get our algorithm SparsiBoost which obtained an optimal trade-off between minimal margin and number of hypotheses. While minimal margin might be insufficient for predicting generalization performance, our sparsification algorithm actually preserves the full distribution of margins. Thus according to margin theory, the sparsified classifier should approximately preserve the generalization performance of the full unsparsified classifier. To demonstrate this experimentally, we sparsified a classifier trained with LightGBM \cite{lightgbm}, a highly efficient open-source implementation of Gradient Boosting \cite{gb1, gb2}. We compared the margin distribution and the test error of the sparsified classifier against a LightGBM classifier trained directly to have the same number of hypotheses. Our results (see Section~\ref{sec:experiments}) show that the sparsified classifier has a better margin distribution and indeed generalize better than the standard LightGBM classifier.

%% file: upperbound.tex
\section{SparsiBoost}
\label{sec:upper}
In this section we introduce SparsiBoost. The algorithm takes the following inputs: training data $D$, a target number of hypotheses $T$ and a base learning algorithm $A$ that returns hypotheses from a class $\mathcal{H}$ of possible base hypotheses. SparsiBoost initially trains $c\cdot T$ hypotheses for some appropriate $c$, by running AdaBoostV with the base learning algorithm $A$. It then removes the extra $c\cdot T-T$ hypotheses while attempting to preserve the margins on all training examples. 

In more detail, let $h_1,...,h_{cT} \in \mathcal{H}$ be the hypotheses returned by AdaBoostV with weights $w_1,...,w_{cT}$. Construct a margin matrix $U$ that contains the margin of every hypothesis $h_j$ on every point $x_i$ such that $u_{ij}=y_i h_j(x_i)$. Let $w$ be the vector of hypothesis weights, meaning that the $j$'th coordinate of $w$ has the weight $w_j$ of hypothesis $h_j$. Normalize $w=w/\|w\|_1$ such that $\|w\|_1=1$. The product $Uw$ is then a vector that contains the margins of the final linear combination on all points: $(Uw)_i=y_i\sum_{j=1}^{cT}w_j h_j(x_i)=\margin(x_i)$. Removing hypotheses while preserving the margins can be formulated as sparsifying $w$ to $w'$ while minimizing $\|Uw-Uw'\|_\infty$ subject to $\|w'\|_0\le T$ and $\|w'\|_1=1$.

See \Cref{algo:sparsi_boost} for pseudocode.
\begin{algorithm}[h]
\caption{SparsiBoost}\label{algo:sparsi_boost}
\textbf{Input:} Training data $D=\{(x_i, y_i)\}_{i=1}^n$ where $x_i\in X$ for some input space $X$ and $y_i\in \{-1,+1\}$.
Target number of hypotheses $T$ and base learning algorithm $A$. \\
\textbf{Output:} Hypotheses $h_1,\dots,h_k$ and weights $w_1,\dots,w_k$ with $k \leq T$, such that $\sum_i w_i h_i$ has gap $O(\sqrt{\lg(2 + n/T)/T})$ on $D$.\\

\textbf{1.} \hspace{1pt}Run AdaBoostV with base learning algorithm $A$ on training data $D$ to get $cT$ hypotheses $h_1,\dots,h_{cT}$ and weights $w_1,\dots,w_{cT}$ for the integer $c = \lceil \log(n)/\log(2 + n/T)\rceil$.\\
\textbf{2.} \hspace{1pt}Construct margin matrix $U\in[-1, +1]^{n\times cT}$ where $u_{ij}=y_ih_j(x_i)$. \\
\textbf{3.} \hspace{1pt}Form the vector $w$ with $i$'th coordinate $w_i$ and normalize $w \gets w/\|w\|_1$ so $\|w\|_1=1$. \\
\textbf{4.} \hspace{1pt}Find $w'$ such that $\|w'\|_0\le T$, $\|w'\|_1=1$ and $\|Uw-Uw'\|_\infty=O(\sqrt{\log(2 + n/T)/T})$. \\
\textbf{5.} \hspace{1pt}Let $\pi(j)$ denote the index of the $j$'th non-zero entry of $w'$.\\
\textbf{6.} \hspace{1pt}\textbf{Return} hypotheses $h_{\pi(1)},\dots,h_{\pi(\|w'\|_0)}$ with weights $w'_{\pi(1)},\dots,w'_{\pi(\|w'\|_0)}$.\\
\end{algorithm}\\
We still haven't described how to find $w'$ with the guarantees shown in \Cref{algo:sparsi_boost} step 4., i.e. a $w'$ with $\|Uw-Uw'\|_\infty=O(\sqrt{\log(2 + n/T)/T})$. It is not even clear that such $w'$ exists, much less so that it can be found efficiently. Before we dive into the details of how to find $w'$, we briefly demonstrate that indeed such a $w'$ would be sufficient to establish our main theorem:
\begin{theorem}
SparsiBoost is guaranteed to find a linear combination $w'$ of at most $T$ base hypotheses with gap $v=O(\sqrt{\log (2 + n/T)/T})$.
\label{thm:upper_bound}
\end{theorem}
\begin{proof}
We assume throughout the proof that a $w'$ with the guarantees claimed in \Cref{algo:sparsi_boost} can be found. Suppose we run AdaBoostV to get $cT$ base hypotheses $h_1,...,h_{cT}$ with weights $w_1,...,w_{cT}$. Let $\rho_{cT}$ be the minimal margin of the linear combination $\sum_i w_i h_i$ on the training data $D$, and let $\rho^*$ be the optimal minimal margin over all linear combinations of base hypotheses from $\mathcal{H}$. As proved in \cite{ratsch2005efficient}, AdaBoostV guarantees that the gap is bounded by 
$
\rho^* - \rho_{cT} = O( \sqrt{\log(n)/(cT)})
$.
Normalize $w=w/\|w\|_1$ and let $U$ be the margin matrix $u_{ij}=y_ih_j(x_i)$ as  in \Cref{algo:sparsi_boost}. Then $\rho_{cT}=\min_i(Uw)_i$. From our assumption, we can efficiently find $w'$ such that $\|w'\|_1=1$, $\|w'\|_0\le T$ and $\|Uw-Uw'\|_\infty = O(\sqrt{\log(2 + n/T)/T})$. Consider the hypotheses that correspond to the non-zero entries of $w'$. There are at most $T$.  Let $\rho_T$ be their minimal margin when using the corresponding weights from $w'$. Since $w'$ has unit $\ell_1$-norm, it follows that $\rho_T = \min_i(Uw')_i$ and thus $|\rho_T - \rho_{cT}| \leq \max_i |(Uw)_i - (Uw')_i|$, i.e.
$
	|\rho_{cT}-\rho_T| \leq \|Uw-Uw'\|_\infty = O(\sqrt{\log(2 + n/T)/T})
$.
We therefore have:
\begin{align*}
	\rho^* -\rho_T
	 = (\rho^* - \rho_{cT}) +  (\rho_{cT} - \rho_T) \leq \\ O(\sqrt{\log(n)/(cT)})+O(\sqrt{\log(2 + n/T)/T}).
\end{align*}
By choosing $c=\log(n)/\log(2 + n/T)$ (as in \Cref{algo:sparsi_boost}) we get that $\rho^*-\rho_T=O(\sqrt{\log(2 + n/T)/T})$. 
\end{proof}

The core difficulty in our algorithm is thus finding an appropriate $w'$ (step 4 in \Cref{algo:sparsi_boost}) and this is the focus of the remainder of this section. Our algorithm for finding $w'$ gives a general method for sparsifying a vector $w$ while approximately preserving every coordinate of the matrix-vector product $Uw$ for some input matrix $U$. The guarantees we give are stated in the following theorem:

\begin{theorem} \label{lemma:margin_preservation} (Sparsification Theorem) 
For all matrices $U\in[-1, +1]^{n\times m}$, all $w\in\R^{m}$ with $\|w\|_1=1$ and all $T \leq m$, there exists a vector $w'$ where $\|w'\|_1=1$ and $\|w'\|_0\le T$, such that $\|Uw-Uw'\|_\infty = O(\sqrt{\log(2 + n/T)/T})$.
\end{theorem}
Theorem~\ref{lemma:margin_preservation} is exactly what was needed in the proof of Theorem~\ref{thm:upper_bound}. Our proof of \Cref{lemma:margin_preservation} will be constructive in that it gives an algorithm for finding $w'$. To keep the proof simple, we will argue about running time at the end of the section.

The first idea in our algorithm and proof of Theorem~\ref{lemma:margin_preservation}, is to reduce the problem to a simpler task, where instead of reducing the number of hypotheses directly to $T$, we only halve the number of hypotheses:
\begin{lemma} \label{lemma:margin_preservation_half} 
For all matrices $U\in[-1, +1]^{n\times m}$ and $w\in\R^{m}$ with $\|w\|_1=1$, there exists $w'$ where $\|w'\|_0\le \|w\|_0/2$ and $\|w'\|_1= 1$, such that $\|Uw-Uw'\|_\infty =O(\sqrt{\log(2 + n/\|w\|_0)/\|w\|_0})$.
\end{lemma}
To prove Theorem~\ref{lemma:margin_preservation} from Lemma~\ref{lemma:margin_preservation_half}, we can repeatedly apply Lemma~\ref{lemma:margin_preservation_half} until we are left with a vector with at most $T$ non-zeroes. Since the loss $O(\sqrt{\log(2 + n/\|w\|_0)/\|w\|_0})$ has a $\sqrt{1/\|w\|_0}$ factor, we can use the triangle inequality to conclude that the total loss is a geometric sum that is asymptotically dominated by the very last invocation of the halving procedure. Since the last invocation has $\|w\|_0 > T$ (otherwise we would have stopped earlier), we get a total loss of $O(\sqrt{\lg(2 + n/T)/T})$ as desired. The formal proof can be found in \Cref{sec:halving}. 

The key idea in implementing the halving procedure \Cref{lemma:margin_preservation_half} is as follows: Let $\pi(j)$ denote the index of the $j$'th non-zero in $w$ and let $\pi^{-1}(j)$ denote the index $i$ such that $w_i$ is the $j$'th non-zero entry of $w$. First we construct a matrix $A$ where the $j$'th column of $A$ is equal to the $\pi(j)$'th column of $U$ scaled by the weight $w_{\pi(j)}$. The sum of the entries in the $i$'th row of $A$ is then equal to the $i$'th entry of $Uw$ (since $\sum_j a_{ij}=\sum_j w_{\pi^{-1}(j)} u_{i \pi^{-1}(j)} = \sum_{j : w_j\neq 0} w_j u_{ij}=\sum_j w_j u_{ij}=(Uw)_i$). Minimizing $\|Uw-Uw'\|_\infty$ can then be done by finding a subset of columns of $A$ that approximately preserves the row sums. This is formally expressed in \Cref{lemma:RSPL} below. For ease of notation we define $a\pm[x]$ to be the interval $[a-x, a+x]$ an $\pm[x]$ to be the interval $[-x, x]$. 
\begin{lemma}\label{lemma:RSPL}
	For all matrices  $A\in[-1, 1]^{n\times T}$ there exists a submatrix $\hat{A}\in[-1,1]^{n\times k}$ consisting of $k\le T/2$ distinct columns from $A$, such that for all $i$, it holds that $\sum_{j=1}^k\hat{a}_{ij}\in\frac{1}{2}\sum_{j=1}^Ta_{ij}\pm \left[O(\sqrt{T\log (2 + n/T)})\right]$.
\end{lemma}
Intuitively we can now use \Cref{lemma:RSPL} to select a subset $S$ of at most $T/2$ columns in $A$. We can then replace the vector $w$ with $w'$ such that $w'_i = 2w_i$ if $i = \pi^{-1}(j)$ for some $j \in S$ and $w'_i = 0$ otherwise. In this way, the $i$'th coordinate $(Uw')_i$ equals the $i$'th row sum in $\hat{A}$, scaled by a factor two. By \Cref{lemma:RSPL}, this in turn approximates the $i$'th row sum in $A$ (and thus $(Uw)_i$) up to additively $O(\sqrt{T \lg (2 + n/T)})$.

Unfortunately our procedure is not quite that straightforward since $O(\sqrt{T \lg(2 + n/T)})$ is way too large compared to $O(\sqrt{\lg(2+n/\|w\|_0)/\|w\|_0}) = O(\sqrt{\lg(2+n/T)/T}) $. Fortunately \Cref{lemma:RSPL} only needs the coordinates of $A$ to be in $[-1, 1]$.  We can thus scale $A$ by $1/\max_i |w_i|$ and still satisfy the constraints. This in turn means that the loss is scaled down by a factor $\max_i |w_i|$. However, $\max_i |w_i|$ may be as large as $1$ for highly unbalanced vectors. Therefore, we start by copying the largest $T/3$ entries of $w$ to $w'$ and invoke \Cref{lemma:RSPL} twice on the remaining $2T/3$ entries. This ensures that the $O(\sqrt{T \lg (2 + n/T)})$ loss in \Cref{lemma:RSPL} gets scaled by a factor at most $3/T$ (since $\|w\|_1=1$, all remaining coordinates are less than or equal to $3/T$), while leaving us with at most $T/3 + (2T/3)/4 = T/3 + T/6 = T/2$ non-zero entries as required. Since we normalize by at most $3/T$, the error becomes $\|Uw-Uw'\|_\infty = O(\sqrt{T \lg (2  +n/T)}/T) = O(\sqrt{\lg(2 + n/T)/T})$ as desired. As a last technical detail, we also need to ensure that $w'$ satisfies $\|w'\|_1=1$. We do this by adding an extra row to $A$ such that $a_{(n+1) j} = w_j$. In this way, preserving the last row sum also (roughly) preserves the $\ell_1$-norm of $w$ and we can safely normalize $w'$ as $w' \gets w'/\|w\|_1$. The formal proof is given in \Cref{sec:sum_preserve}. 

The final step of our algorithm is thus to select a subset of at most half of the columns from a matrix $A \in [-1 , 1]^{n \times T}$, while approximately preserving all row sums. Our idea for doing so builds on the following seminal result by Spencer:
\begin{theorem}(Spencer's Theorem~\cite{spencer1985six})
\label{thm:spencer}
 For all matrices $A\in [-1, +1]^{n \times T}$ with $T \leq n$, there exists $x\in\{-1,+1\}^{T}$ such that $\|Ax\|_\infty = O(\sqrt{T \ln (en/T)})$. For all matrices $A \in [-1,+1]^{n \times T}$ with $T > n$, there exists $x \in \{-1,+1\}^T$ such that $\|Ax\|_\infty = O(\sqrt{n})$.
\end{theorem}
We use Spencer's Theorem as follows: We find a vector $x \in \{-1,+1\}^T$ with $\|Ax\|_\infty = O(\sqrt{T \ln(en/T)})$ if $T \leq n$ and with $\|Ax\|_\infty = O(\sqrt{n}) = O(\sqrt{T})$ if $T > n$. Thus we always have $\|Ax\|_\infty = O(\sqrt{T \lg(2 + n/T)})$. Consider now the $i$'th row of $A$ and notice that $|\sum_{j : x_j=1} a_{ij} - \sum_{j : x_j=-1} a_{ij}| \leq \|Ax\|_\infty$. That is, for every single row, the sum of the entries corresponding to columns where $x$ is $1$, is almost equal (up to $\pm \|Ax\|_\infty$) to the sum over the columns where $x$ is $-1$. Since the two together sum to the full row sum, it follows that the subset of columns with $x_i=1$ \emph{and} the subset of columns with $x_i=-1$ \emph{both} preserve the row sum as required by \Cref{lemma:RSPL}. Since $x$ has at most $T/2$ of either $+1$ or $-1$, it follows that we can find the desired subset of columns. We give the formal proof of \Cref{lemma:RSPL} using Spencer's Theorem in \Cref{sec:use_spencer}

We have summarized the entire sparsification algorithm in Algorithm~\ref{algo:sparsify}.

\begin{algorithm}[h]
\caption{Sparsification}\label{algo:sparsify}
\textbf{Input:} Matrix $U \in [-1, 1]^{n \times m}$, vector $w \in \R^{m}$ with $\|w\|_1=1$ and target $T \leq m$.\\
\textbf{Output:} A vector $w' \in \R^m$ with $\|w'\|_1=1, \|w'\|_0 \leq T$ and $\|Uw - Uw'\|_\infty = O(\sqrt{\lg(2 + n/T)/T})$.\\

\textbf{1.} \hspace{1pt} Let $w' \gets w$.\\
\textbf{2.} \hspace{1pt} \textbf{While} $\|w'\|_0 > T$: \\
\textbf{3.} \hspace{15pt} Let $R$ be the indices of the $\|w'\|_0/3$ entries in $w'$ with largest absolute value.\\
\textbf{4.} \hspace{15pt} Let $\omega := \max_{i \notin R} |w_{i}|$ be the largest value of an entry outside $R$.\\
\textbf{5.} \hspace{15pt} \textbf{Do Twice}:\\
\textbf{6.} \hspace{30pt} Let $\pi(1),\pi(2),\dots,...,\pi(k)$ be the indices of the non-zero entries in $w'$ that are not in $R$.\\
\textbf{7.} \hspace{30pt} Let $A \in [-1,1]^{(n+1) \times k}$ have $a_{ij} = u_{i \pi(j)} w'_{\pi(j)} / \omega$ for $i \leq n$ and $a_{(n+1)j} = |w'_{\pi(j)}|/\omega$.\\
\textbf{8.} \hspace{30pt} Invoke Spencer's Theorem to find $x \in \{-1,1\}^k$ such that $\|Ax\|_\infty = O(\sqrt{k \lg(2 + n/k)})$.\\
\textbf{9.} \hspace{30pt} Let $\sigma \in \{-1,1\}$ denote the sign such that $x_i=\sigma$ for at most $k/2$ indices $i$.\\
\textbf{10.} \hspace{25pt} Update $w'_i$ as follows:\\
\textbf{11.} \hspace{45pt} If there is a $j$ such that $i = \pi^{-1}(j)$ and $x_j=\sigma$: set $w'_i \gets 2w'_i$.\\
\textbf{12.} \hspace{45pt} If there is a $j$ such that $i = \pi^{-1}(j)$ and $x_j\neq \sigma$: set $w'_i \gets 0$.\\
\textbf{13.} \hspace{45pt} Otherwise ($i \in R$ or $w_i'=0$): set $w'_i \gets w'_i$.\\
\textbf{14.} \hspace{12pt} Update $w' \gets w'/\|w'\|_1$.\\
\textbf{15.} \hspace{1pt}\textbf{Return} $w'$.\\
\end{algorithm}
We make a few remarks about our sparsification algorithm and SparsiBoost.

\paragraph{Running Time.}
While Spencer's original result (Theorem~\ref{thm:spencer}) is purely existential, recent follow up work \cite{discmin1} show how to find the vector $x \in \{-1,+1\}^n$ in expected $\tilde{O}((n+T)^3)$ time, where $\tilde{O}$ hides polylogarithmic factors.  A small modification to the algorithm was suggested in \cite{discmin2}. This modification reduces the running time of Lovett and Meka's algorithm to  expected $\tilde{O}(nT + T^3)$. This is far more desirable as $T$ tends to be much smaller than $n$ in boosting. Moreover, the $nT$ term is already paid by simply running AdaBoostV. Using this in Step 7. of Algorithm~\ref{algo:sparsify}, we get a total expected running time of $\tilde{O}(nT + T^3)$. We remark that these algorithms are randomized and lead to different vectors $x$ on different executions.

\paragraph{Non-Negativity.}
Examining Algorithm~\ref{algo:sparsify}, we observe that the weights of the input vector are only ever copied, set to zero, or scaled by a factor two. Hence if the input vector $w$ has non-negative entries, then so has the final output vector $w'$. This may be quite important if one interprets the linear combination over hypotheses as a probability distribution.

\paragraph{Importance Sampling.}
Another natural approach one might attempt in order to prove our sparsification result, \Cref{lemma:margin_preservation}, would be to apply importance sampling. Importance sampling samples $T$ entries from $w$ with replacement, such that each entry $i$ is sampled with probability $|w_i|$. It then returns the vector $w'$ where coordinate $i$ is equal to $\sign(w_i)n_i/T$ where $n_i$ denotes the number of times $i$ was sampled and $\sign(w_i) \in \{-1,1\}$ gives the sign of $w_i$. Analysing this method gives a $w'$ with $\|Uw-Uw'\|_\infty = \Theta(\sqrt{\lg(n)/T})$ (with high probability), i.e. slightly worse than our approach based on discrepancy minimization. The loss in the $\lg$ is enough that if we use importance sampling in SparsiBoost, then we get no improvement over simply stopping AdaBoostV after $T$ iterations.

\subsection{Repeated Halving}
\label{sec:halving}
As discussed earlier, our matrix-vector sparsification algorithm (\Cref{lemma:margin_preservation} and Algorithm~\ref{algo:sparsify}) was based on repeatedly invoking \Cref{lemma:margin_preservation_half}, which halves the number of non-zero entries. In this section, we prove formally that the errors resulting from these halving steps are dominated by the very last round of halving. For convenience, we first restate the two results:
\begin{customthm}{\ref{lemma:margin_preservation}}(Sparsification Theorem) 
For all matrices $U\in[-1, +1]^{n\times m}$, all $w\in\R^{m}$ with $\|w\|_1=1$ and all $T \leq m$, there exists a vector $w'$ where $\|w'\|_1=1$ and $\|w'\|_0\le T$, such that $\|Uw-Uw'\|_\infty = O(\sqrt{\log(2 + n/T)/T})$.
\end{customthm}

\begin{customlem}{\ref{lemma:margin_preservation_half}}
For all matrices $U\in[-1, +1]^{n\times m}$ and $w\in\R^{m}$ with $\|w\|_1=1$, there exists $w'$ where $\|w'\|_0\le \|w\|_0/2$ and $\|w'\|_1= 1$, such that $\|Uw-Uw'\|_\infty =O(\sqrt{\log(2 + n/\|w\|_0)/\|w\|_0})$.
\end{customlem}

We thus set out to prove \Cref{lemma:margin_preservation} assuming \Cref{lemma:margin_preservation_half}.

\begin{proof}
Let us call the initial weight vector $w^{(0)}=w$. We use \Cref{lemma:margin_preservation_half} repeatedly and get $w^{(1)},w^{(2)},...,w^{(k)}=w'$ such that $\|w'\|_0\le T$ as wanted and every $w^{(i)}$ has $\|w^{(i)}\|_1=1$. Let $T_i=\|w^{(i)}\|_0$ and notice this gives a sequence of numbers $T_0, T_1, ..., T_k$ where $T_i\le T_{i-1}/2$, $T_{k-1} > T$ and $T_k\le T$. In particular, it holds that $T_{k-1}2^{k-i-1}\le T_{i}$. The total difference $\|Uw^{(0)}-Uw^{(k)}\|_\infty$ is then by the triangle inequality no more than $\sum_{i=0}^{k-1}\|Uw^{(i)}-Uw^{(i+1)}\|_\infty$. Each of these terms are bounded by \Cref{lemma:margin_preservation_half} which gives us 
\begin{align*}
O\left(\sum_{i=0}^{k-1}\sqrt{\log(2 + n/T_i)/T_i} \right) &= 
O\left(\sum_{i=0}^{k-1}\sqrt{\log(2 + n/T_{k-1})/(T_{k-1}2^{k-i-1})}\right) =\\
O\left(\sum_{i=0}^\infty \sqrt{\log(2 + n/T_{k-1})/(T_{k-1}2^{i})}\right) &=
O\left( \left(\sqrt{\log(2 + n/T_{k-1})/T_{k-1}} \right)\sum_{i=0}^\infty  1/\sqrt{2^{i}}\right)=
\\O\left(\sqrt{\log(2 + n/T)/T}\right).
\end{align*}
The last step follows since $T_{k-1} > T$. We have thus shown that the final vector $w' = w^{(k)}$ has $\|w'\|_0 \leq T$, $\|w'\|_1=1$ and $\|Uw-Uw'\|_\infty = \|Uw^{(0)}-Uw^{(k)}\|_\infty = O\left(\sqrt{\log(2 + n/T)/T}\right)$. This completes the proof of \Cref{lemma:margin_preservation}.
\end{proof}

\subsection{Halving via Row-Sum Preservation}
\label{sec:sum_preserve}
In this section, we give the formal details of how to use our result on row-sum preservation to halve the number of non-zeroes in a vector $w$. Let us recall the two results:
\begin{customlem}{\ref{lemma:margin_preservation_half}}
For all matrices $U\in[-1, +1]^{n\times m}$ and $w\in\R^{m}$ with $\|w\|_1=1$, there exists $w'$ where $\|w'\|_0\le \|w\|_0/2$ and $\|w'\|_1= 1$, such that $\|Uw-Uw'\|_\infty =O(\sqrt{\log(2 + n/\|w\|_0)/\|w\|_0})$.
\end{customlem}

\begin{customlem}{\ref{lemma:RSPL}}
	For all matrices  $A\in[-1, 1]^{n\times T}$ there exists a submatrix $\hat{A}\in[-1,1]^{n\times k}$ consisting of $k\le T/2$ distinct columns from $A$, such that for all $i$, it holds that $\sum_{j=1}^k\hat{a}_{ij}\in\frac{1}{2}\sum_{j=1}^Ta_{ij}\pm \left[O(\sqrt{T\log (2 + n/T)})\right]$.
\end{customlem}
We now use \Cref{lemma:RSPL} to prove \Cref{lemma:margin_preservation_half}. 
\begin{proof}
Our procedure corresponds to steps 2.-13. in Algorithm~\ref{algo:sparsify}. To clarify the proof, we expand the steps a bit and introduce some additional notation. Let $w$ be the input vector and let $R$ be the indices of the $\|w\|_0/3$ entries in $w$ with largest absolute value. Define $\bar{w}$ such that $\bar{w}_i = w_i$ if $i \in R$ and $\bar{w}_i = 0$ otherwise, that is, $\bar{w}$ contains the largest $\|w\|_0/3$ entries of $w$ and is zero elsewhere. Similarly, define $\hat{w} = w-\bar{w}$ as the vector containing all but the largest $\|w\|_0/3$ entries of $w$.

We will use \Cref{lemma:RSPL} twice in order to obtain a vector $w''$ with $\|w''\|_0 \leq \|\hat{w}\|_0/4$, $\|w''\|_1 \in \|\hat{w}\|_1 \pm [O(\sqrt{\log(2 + n/\|w\|_0)/\|w\|_0})]$ and $\|U\hat{w} - Uw''\|_\infty = O(\sqrt{\log(2 + n/\|w\|_0)/\|w\|_0})$. Moreover, $w''$ will only be non-zero in entries where $\hat{w}$ is also non-zero. We  finally set $w' = (\bar{w} + w'')/\|\bar{w} + w''\|_1$ as our sparsified vector.

We first argue that if we can indeed produce the claimed $w''$, then $w'$ satisfies the claims in \Cref{lemma:margin_preservation_half}: Observe that $\|w'\|_0 = \|\bar{w}\|_0 + \|w''\|_0 \leq \|w\|_0/3 + (2\|w\|_0/3)/4 \leq \|w\|_0/2$ as desired. Clearly we also have $\|w'\|_1=1$ because of the normalization. Now observe that:
\begin{eqnarray*}
\|Uw-Uw'\|_\infty &\leq& \|Uw-U(\bar{w}+w'')\|_\infty + \|U(\bar{w}+w'') - Uw'\|_\infty \\
&=& \|U(\bar{w} + \hat{w})-U(\bar{w}+w'')\|_\infty + \|U(w'\|\bar{w}+w''\|_1) - Uw'\|_\infty \\
&=& \|U\hat{w}-Uw''\|_\infty + \|Uw'(\|\bar{w}+w''\|_1-1)\|_\infty\\
&=& O(\sqrt{\log(2 + n/\|w\|_0)/\|w\|_0}) + \|Uw'\|_\infty (\|\bar{w}+w''\|_1-1)\\
&=& O(\sqrt{\log(2 + n/\|w\|_0)/\|w\|_0}) + \|Uw'\|_\infty (\|\bar{w}\|_1+\|w''\|_1-1).\\
\end{eqnarray*}
In the last step, we used that $w''$ has non-zeroes only where $\hat{w}$ has non-zeroes (and thus the non-zeroes of $\bar{w}$ and $w''$ are disjoint). Since $\|w'\|_1=1$ and all entries of $U$ are in $[-1,1]$, we get $\|Uw'\|_\infty \leq 1$. We also see that:
\begin{eqnarray*}
\|\bar{w}\|_1+\|w''\|_1-1 &\in& \|\bar{w}\|_1+\|\hat{w}\|_1 -1 \pm  [O(\sqrt{\log(2 + n/\|w\|_0)/\|w\|_0})] \\
&=& \|\bar{w}+\hat{w}\|_1 -1 \pm  [O(\sqrt{\log(2 + n/\|w\|_0)/\|w\|_0})] \\
&=& \|w\|_1 -1 \pm  [O(\sqrt{\log(2 + n/\|w\|_0)/\|w\|_0})]\\
&=& 0\pm  [O(\sqrt{\log(2 + n/\|w\|_0)/\|w\|_0})].
\end{eqnarray*}
We have thus shown that
$$
\|Uw-Uw'\|_\infty = O(\sqrt{\log(2 + n/\|w\|_0)/\|w\|_0})
$$
as claimed. So what remains is to argue that we can find $w''$ with the claimed properties. Finding $w''$ corresponds to the Do Twice part of Algorithm~\ref{algo:sparsify} (steps 3.-12.). We first compute $\omega = \max_i |\hat{w}_i|$ and let $w'' = \hat{w}$. We then execute the following twice:
\begin{enumerate}
\item Define $\pi(1),\dots,\pi(\ell)$ as the list of all the indices $i$ where $w''_i \neq 0$. Also define $\pi^{-1}(j)$ as the index $i$ such that $i = \pi(j)$, i.e. $\pi^{-1}(j)$ is the index of the $j$'th non-zero coordinate of $w''$.
\item Form the matrix $A \in [-1,1]^{(n+1)\times \ell}$ where $a_{ij} = u_{i \pi(j)} w''_{\pi(j)}/\omega$ for $i\leq n$ and $a_{(n+1)j} = |w''_{\pi(j)}|/\omega$.
\item Invoke \Cref{lemma:RSPL} to obtain a matrix $\hat{A}$ consisting of no more than $k \leq \ell/2$ distinct columns from $A$ where for all rows $i$, we have $\sum_{j=1}^k \hat{a}_{ij} \in \frac{1}{2}\sum_{j=1}^\ell a_{ij}\pm \left[O(\sqrt{T\log (2 + n/T)})\right]$.
\item Update $w''$ as follows: 
\begin{enumerate}
\item We let $w''_i \gets 2w''_i$ if there is a $j$ such that $i = \pi^{-1}(j)$ and $j$ is a column in $\hat{A}$.
\item Otherwise we let $w''_i \gets 0$.
\end{enumerate}
\end{enumerate}
After the two executions, we get that the number of non-zeroes in $w''$ is at most $\|\hat{w}\|_0/4$ as claimed. Step 4. above effectively scales every coordinate of $w''$ that corresponds to a column that was included in $\hat{A}$ by a factor two. It sets coordinates not chosen by $\hat{A}$ to $0$. What remains is to argue that $\|U\hat{w} - Uw''\|_\infty = O(\sqrt{\log(2 + n/\|w\|_0)/\|w\|_0})$ and that $\|w''\|_1  \in \|\hat{w}\|_1 \pm [O(\sqrt{\log(2 + n/\|w\|_0)/\|w\|_0})]$. For this, let $A \in [-1,1]^{(n+1) \times T}$ denote the matrix formed in step 2. during the first iteration. Then $T = \|\hat{w}\|_0$. Let $\bar{A} \in [-1,1]^{(n+1) \times k}$ denote the matrix returned in step 3. of the first iteration. Similarly, let $\hat{A}$ denote the matrix formed in step 2. of the second iteration and let $\hat{\hat{A}} \in [-1,1]^{(n+1) \times k'}$ denote the matrix returned in step 3. of the second iteration. We see that $\hat{A} = 2 \bar{A}$. By \Cref{lemma:RSPL}, it holds for all rows $i$ that:
\begin{eqnarray*}
2\sum_{j=1}^{k'}\hat{\hat{a}}_{ij} &\in& \sum_{j=1}^k \hat{a}_{ij} \pm \left[O(\sqrt{k\log(2 + n/k)})\right]\\
&=& 2 \sum_{j=1}^k \bar{a}_{ij} \pm \left[O(\sqrt{k\log(2 + n/k)})\right] \\
&\subseteq& \left(\sum_{j=1}^T a_{ij} \pm   \left[O(\sqrt{k\log(2 + n/k)})\right] \right) \pm  \left[O(\sqrt{T\log(2 + n/T)})\right] \\
&\subseteq& \sum_{j=1}^T a_{ij} \pm   \left[O(\sqrt{T\log(2 + n/T)})\right].
\end{eqnarray*}
But $2\sum_{j=1}^{k'}\hat{\hat{a}}_{ij} = \sum_j u_{ij} w''_j/\omega = (Uw'')_i/\omega$ and $\sum_{j=1}^T a_{ij} = \sum_j u_{ij} \hat{w}_j/\omega = (U\hat{w})_i/\omega$. Hence:
\begin{eqnarray*}
(Uw'')_i/\omega &\in& (U\hat{w})_i/\omega \pm \left[O(\sqrt{T\log(2 + n/T)})\right] \Rightarrow \\
(Uw'')_i &\in&  (U\hat{w})_i\pm \left[O(\omega \sqrt{T\log(2 + n/T)})\right]
\end{eqnarray*}
which implies that $\|U\hat{w}-Uw''\|_\infty = O(\omega \sqrt{T\log(2 + n/T)})$. But $T = \|\hat{w}\|_0 \leq \|w\|_0$ and $\omega = \max_i |\hat{w}|_i$. Since $\|w\|_1=1$ and $\bar{w}$ contained the largest $\|w\|_0/3$ entries, we must have $\omega \leq 3/\|w\|_0$. Inserting this, we conclude that
$$
\|U\hat{w}-Uw''\|_\infty = O(\sqrt{\|w\|_0 \log(2 + n/\|w\|_0)}/\|w\|_0) = O(\sqrt{\lg(2 + n/\|w\|_0)/\|w\|_0}).
$$
The last step is to prove that $\|w''\|_1 \in \|\hat{w}\|_1 \pm [O(\sqrt{\log(2 + n/\|w\|_0)/\|w\|_0})]$. Here we focus on row $(n+1)$ of $A$ and use that we showed above that $2\sum_{j=1}^{k'}\hat{\hat{a}}_{(n+1)j} = \sum_{j=1}^T a_{(n+1)j} \pm   \left[O(\sqrt{T\log(2 + n/T)})\right]$. This time, we have $\sum_{j=1}^T a_{(n+1)j} = \sum_{j=1}^T |\hat{w}_j|/\omega = \|\hat{w}\|_1/\omega$ and $2 \sum_{j=1}^{k'} \hat{\hat{a}}_{(n+1)j} = \sum_j |w''_j|/\omega = \|w''\|_1/\omega$. Therefore we conclude that $\|w''\|_1 \in \|\hat{w}\|_1 \pm \left[O(\omega \sqrt{T\log(2 + n/T)})\right] \subseteq \|\hat{w}\|_1 \pm [O(\sqrt{\log(2 + n/\|w\|_0)/\|w\|_0})]$ as claimed.
\end{proof}

\subsection{Finding a Column Subset}
\label{sec:use_spencer}
In this section we give the detailed proof of how to use Spencer's theorem to select a subset of columns in a matrix while approximately preserving its row sums. The two relevant lemmas are restated here for convenience:
\begin{customlem}{\ref{lemma:RSPL}}
	For all matrices  $A\in[-1, 1]^{n\times T}$ there exists a submatrix $\hat{A}\in[-1,1]^{n\times k}$ consisting of $k\le T/2$ distinct columns from $A$, such that for all $i$, it holds that $\sum_{j=1}^k\hat{a}_{ij}\in\frac{1}{2}\sum_{j=1}^Ta_{ij}\pm \left[O(\sqrt{T\log (2 + n/T)})\right]$.
\end{customlem}
and
\begin{customthm}{\ref{thm:spencer}} (Spencer's Theorem~\cite{spencer1985six})
 For all matrices $A\in [-1, +1]^{n \times T}$ with $T \leq n$, there exists $x\in\{-1,+1\}^{T}$ such that $\|Ax\|_\infty = O(\sqrt{T \ln (en/T)})$. For all matrices $A \in [-1,+1]^{n \times T}$ with $T > n$, there exists $x \in \{-1,+1\}^T$ such that $\|Ax\|_\infty = O(\sqrt{n})$.
\end{customthm}
We now prove \Cref{lemma:RSPL} using Spencer's Theorem:
\begin{proof}Let $A\in[-1,1]^{n\times T}$ as in \Cref{lemma:RSPL} and let $a_i$ denote the $i$'th row of $A$. By Spencer's Theorem there must exist $x\in\{-1,+1\}^T$ s.t. $\|Ax\|_\infty=\max_i |\langle a_i, x \rangle| = O(\sqrt{T\ln(en/T)})$ if $T \leq n$ and $\|Ax\|_\infty = O(\sqrt{n}) = O(\sqrt{T})$ if $T > n$. This ensures that $\|Ax\|_\infty = O(\sqrt{T \lg(2 + n/T)})$ Either the number of $+1$'s or $-1$'s in $x$ will be $\le \frac{T}{2}$. Assume without loss of generality that the number of $+1$'s is $k\le \frac{T}{2}$. Construct a matrix $\hat{A}$ with columns of $A$ corresponding to the entries of $x$ that are $+1$. Then $\hat{A}\in[-1,1]^{n\times k}$ for $k\le \frac{T}{2}$ as wanted. It is then left to show that $\hat{A}$ preserves the row sums. For all $i$, we have:
\begin{align*}
	|\langle a_i,x\rangle|=\left|\sum_{j=1}^T a_{ij}x_j\right|= O(\sqrt{T\log(2 + n/T)})
	\quad&\Rightarrow\quad
	\sum_j a_{ij}x_j \in \pm\left[O(\sqrt{T\log(2 + n/T)})\right]\\
	&\Rightarrow \sum_{j:x_j=+1}a_{ij} - \sum_{j:x_j=-1}a_{ij}\in\pm\left[ O(\sqrt{T\log(n/T)})\right]\\
	&\Rightarrow \sum_{j:x_j=+1} a_{ij}\in \sum_{j:x_j:-1} a_{ij} \pm\left[ O(\sqrt{T\log(n/T)})\right].
\end{align*}
Using this we can bound the $i$'th row sum of $\hat{A}$
\begin{align*}
	\sum_{j=1}^k \hat{a}_{ij}=\sum_{j: x_j=+1}a_{ij}
		&= \frac{1}{2}\sum_{j: x_j=+1}a_{ij}+\frac{1}{2}\sum_{j: x_j=+1}a_{ij}\\
		&\in \frac{1}{2}\sum_{j: x_j=+1}a_{ij} + \frac{1}{2}
				\left( \sum_{j: x_j=-1}a_{ij} \pm 
						\left[
								O(\sqrt{T \log(2 + n/T)})
						\right] 
				\right) \\
		&=\frac{1}{2}\sum_j a_{ij}\pm \left[O(\sqrt{T\log(2 + n/T)})\right].
	\end{align*}
This concludes the proof. 
\end{proof}

%% file: inequality.tex
\section{Lower Bound}
In this section, we prove our lower bound stating that there exist
a data set and corresponding set of base hypotheses $\mathcal{H}$, such that if one uses
only $T$ of the base hypotheses in $\mathcal{H}$, then one cannot obtain a gap smaller than
$\Omega(\sqrt{\log(n/T)/T})$. Similarly to the approach taken
in~\cite{nie2013open}, we model a data set $D=\{(x_i,y_i)\}_{i=1}^n$ of $n$ data points and a corresponding
set of $k$ base hypotheses $\mathcal{H} = \{h_1,\dots,h_k\}$ as an $n \times k$ matrix
$A$. The entry $a_{i,j}$ is equal to $y_i h_j(x_i)$. We prove our
lower bound for binary classification where the hypotheses take values
only amongst $\{-1,+1\}$, meaning that $A \in \{-1,+1\}^{n \times
  k}$. Thus an entry $a_{i,j}$ is $+1$ if hypothesis $h_j$ is correct
on point $x_i$ and it is $-1$ otherwise. We remark that proving the lower
bound under the restriction that $h_j(x_i)$ is among $\{-1, +1\}$
instead of $[-1,+1]$ only strengthens the lower bound.

Notice that if $w \in \R^k$ is a vector with $\|w\|_1=1$, then $(Aw)_i$ gives exactly the margin
on data point $(x_i,y_i)$ when using the linear combination $\sum_j
w_j h_j$ of base hypotheses. The optimal minimum margin $\rho^*$ for a matrix
$A$ is thus equal to
$
\rho^* := \max_{w \in \R^k : \|w\|_1 = 1} \min_i (Aw)_i.
$
We now seek a matrix $A$ for which $\rho^*$ is at least $\Omega(\sqrt{\log(n/T)/T})$ larger than  $\min_i (Aw)_i$
for all $w$ with $\|w\|_0 \leq T$ and $\|w\|_1 = 1$.
If we can find such a matrix, it implies the existence of a data
set (rows) and a set of base hypotheses (columns) for which any linear
combination of up to $T$ base hypotheses has a gap of
$\Omega(\sqrt{\log(n/T)/T})$. The lower bound thus holds regardless of
how an algorithm would try to determine which linear combination to
construct.

When showing the existence of a matrix $A$ with a large gap, we will
fix $k = n$, i.e. the set of base hypotheses $\mathcal{H}$ has
cardinality equal to the number
of data points. The following theorem shows the existence of the desired matrix $A$:

\begin{theorem}
\label{thm:gap}
There exists a universal constant $C > 0$ such that for all sufficiently large
$n$ and all $T$ with $\ln n \leq T \leq n/C$, there exists a matrix
$A \in \{-1,+1\}^{n \times n}$ such that: 1) Let $v \in \R^n$ be the vector with all coordinates equal to
  $1/n$. Then all coordinates of $Av$
are greater than or equal to $-O(1/\sqrt{n})$. 2) For every vector $w
\in \R^n$ with $\|w\|_0 \leq T$ and $\|w\|_1 = 1$, it holds that:
$
\min_i (Aw)_i \leq -\Omega\left( \sqrt{\log(n/T)/T}\right).
$ 
\end{theorem}

Quite surprisingly, Theorem~\ref{thm:gap} shows that for any $T$ with $\ln n
\leq T \leq n/C$,  there is a matrix
$A \in \{-1,+1\}^{n \times n}$ for which the \emph{uniform}
combination of base hypotheses $\sum_{j=1}^n h_j / n$ has a minimum
margin that is much higher than anything that can be obtained using
only $T$ base hypotheses. More concretely, let $A$ be a matrix
satisfying the properties of Theorem~\ref{thm:gap} and let $v \in
\R^n$ be the vector with all coordinates $1/n$. Then 
$
\rho^* := \max_{w \in \R^k : \|w\|_1 = 1} \min_i (Aw)_i \geq \min_i
(Av)_i = -O(1/\sqrt{n}).
$
By the second property in Theorem~\ref{thm:gap}, it follows that any
linear combination of at most $T$ base hypotheses must have a gap of
$
-O(1/\sqrt{n}) - \left(-\Omega\left(\sqrt{\log(n/T)/T}\right)\right) = \Omega\left(\sqrt{\log(n/T)/T}\right).
$
This is precisely the claimed lower bound and also shows that our
vector sparsification algorithm from \Cref{lemma:margin_preservation} is optimal for any $T
\leq n/C$. To prove
Theorem~\ref{thm:gap}, we first show the existence of a matrix $B \in
\{-1,+1\}^{n \times n}$ having the second property. We then apply
Spencer's Theorem (Theorem~\ref{thm:spencer}) to
``transform'' $B$ into a matrix $A$ having both properties. We find it
quite surprising that Spencer's discrepancy minimization result finds
applications in both our upper and lower bound.

That a matrix satisfying the second property in Theorem~\ref{thm:gap}
exists is expressed in the following lemma:

\begin{lemma}
\label{lem:smallcoord}
There exists a universal constant $C > 0$ such that for all
sufficiently large $n$ and all $T$ with $\ln n \leq T \leq n/C$,
there exists a matrix $A \in \{-1,+1\}^{n \times n}$ such that for
every vector $w \in \R^n$ with $\|w\|_0 \leq T$ and $\|w\|_1 = 1$ it
holds that:
$
\min_i (Aw)_i \leq -\Omega\left( \sqrt{\ln(n/T)/T}\right).
$
\end{lemma}

We prove Lemma~\ref{lem:smallcoord} in later subsection and move on to show how we use it in combination with Spencer's Theorem to prove Theorem~\ref{thm:gap}:

\begin{proof}
Let $B$ be a matrix satisfying the statement in
Lemma~\ref{lem:smallcoord}. Using Spencer's Theorem (Theorem~\ref{thm:spencer}), we get that there exists a vector $x \in
\{-1,+1\}^n$ such that $\|B x \|_\infty = O(\sqrt{n \ln(e n/n)}) = O(\sqrt{n})$. Now form the
matrix $A$ which is equal to $B$, except that the $i$'th column is
scaled by $x_i$. Then $A\textbf{1} = Bx$ where $\textbf{1}$ is the
all-ones vectors. Normalizing the all-ones vector by a factor $1/n$
yields the vector $v$ with all coordinates equal to $1/n$. Moreover,
it holds that $\|Av\|_\infty = \|B x\|_\infty/n = O(1/\sqrt{n})$,
which in turn implies that $\min_i (Av)_i \geq -O(1/\sqrt{n})$.

Now consider any vector $w \in \R^n$ with $\|w\|_0 \leq T$ and
$\|w\|_1 = 1$. Let $\tilde{w}$ be the vector obtained from $w$ by
multiplying $w_i$ by $x_i$. Then $Aw = B\tilde{w}$. Furthermore
$\|\tilde{w}\|_1 = \|w\|_1 = 1$ and $\|\tilde{w}\|_0  = \|w\|_0 \leq T$. It
follows from Lemma~\ref{lem:smallcoord} and our choice of $B$ that
$
\min_i  (Aw)_i = \min_i (B \tilde{w})_i \leq -\Omega\left( \sqrt{\ln(n/T)/T}\right).
$
\end{proof}

The proof of Lemma~\ref{lem:smallcoord} is deferred to the following subsections, here we
sketch the ideas. At a high level, our proof goes as
follows: First argue that for a fixed vector $w \in \R^n$ with
$\|w\|_0 \leq T$ and $\|w\|_1=1$ and a random matrix $A \in
\{-1,+1\}^{n \times n}$ with each coordinate chosen uniformly and independently, it
holds with very high probability that $Aw$ has many coordinates that
are less than $-a_1( \sqrt{\ln(n/T)/T})$ for a constant $a_1$. Now intuitively
we would like to union bound over all possible vectors $w$ and argue
that with non-zero probability, all of them satisfies this simulatenously. This is not directly possible as there are infinitely
many $w$. Instead, we create a \emph{net} $W$ consisting of a collection
of carefully chosen vectors. The net will have the property that any
$w$ with $\|w\|_1=1$ and $\|w\|_0 \leq T$ will be close to a vector
$\tilde{w} \in W$. Since the net is not too large, we can union bound
over all vectors in $W$ and find a matrix $A$ with the above
property for all vectors in $W$ simultaneously.

For an arbitrary vector $w$ with $\|w\|_1 = 1$ and $\|w\|_0 \leq T$, we can then write $Aw = A\tilde{w} +
A(w-\tilde{w})$ where $\tilde{w} \in W$ is close to $w$. Since $\tilde{w} \in W$ we get that $A\tilde{w}$ has
many coordinates that are less than $-a_1( \sqrt{\ln(n/T)/T})$. The
problem is that $A(w-\tilde{w})$ might cancel out these negative
coordinates. However, $w-\tilde{w}$ is a very short vector so this
seems unlikely. To prove it formally, we further show that for
every vector $v \in W$, there are also few coordinates in $Av$ that are greater
than $a_2( \sqrt{\ln(n/T)/T})$ in absolute value for some constant
$a_2 > a_1$. We can then round
$(w-\tilde{w})/\|w-\tilde{w}\|_1$ to a vector in the net, apply this reasoning and recurse
on the difference between $(w-\tilde{w})/\|w-\tilde{w}\|_1$ and the net
vector. We give the full details in the following subsections.

\subsection{Preliminaries}
In the following, we will introduce a few tail bounds that will be
necessary in our proof of Lemma~\ref{lem:smallcoord}.

\begin{definition}
For a vector $w \in \R^n$, define $F(w,t)$ as follows: Let $i_j$ be the index such that $w_{i_j}$ is the $j$'th largest
coordinate of $w$ in terms of absolute value. Then:
$$
F(w,t) := \left(\sum_{j=1}^{\lfloor
    t^2 \rfloor} |w_{i_j}| + t \left(\sum_{j=\lfloor t^2 \rfloor +
      1}^n w_{i_j}^2\right)^{1/2}\right).
$$
\end{definition}

We need the following tail upper and lower bounds for the distribution of $\langle a, x \rangle$ where $a \in \R^n$ and $x \in \{-1,+1\}^n$ has uniform random and independent Rademacher coordinates:

\begin{theorem}[{\cite{rademacherbounds}}]
\label{thm:tailbounds}
There exists universal constants $c_1,c_2,c_3 > 0$ such that the
following holds: For any vector $w \in \R^n$, if $x \in \{-1,+1\}^n$
has uniform random and independent Rademacher coordinates, then:
$$
\forall t > 0 : \Pr\left[\langle w, x\rangle > c_1 F(w,t) \right] \leq e^{-t^2/2}
$$
and
$$
\forall t > 0 : \Pr\left[\langle w, x\rangle > c_2 F(w,t) \right] \geq c_3^{-1} e^{-c_3 t^2}.
$$
\end{theorem}

\begin{lemma}
\label{lem:Fineq}
For any vector $w \in \R^n$, integer $t > 0$ and any $a \geq 1$, it holds that:
$$
F(w,a t) \leq a^2 F(w,t)
$$
\end{lemma}

\begin{proof}
By definition, we see that:
\begin{eqnarray*}
F(w, at) &=& \left(\sum_{j=1}^{\lfloor
    a^2 t^2 \rfloor} |w_{i_j}| + at \left(\sum_{j=\lfloor a^2t^2 \rfloor +
      1}^n w_{i_j}^2\right)^{1/2}\right) \\
&\leq& \left(a^2 \sum_{j=1}^{\lfloor
     t^2 \rfloor} |w_{i_j}| + a^2 t \left(\sum_{j=\lfloor a^2t^2 \rfloor +
      1}^n w_{i_j}^2\right)^{1/2}\right) \\
&\leq& \left(a^2 \sum_{j=1}^{\lfloor
     t^2 \rfloor} |w_{i_j}| + a^2 t \left(\sum_{j=\lfloor t^2 \rfloor +
      1}^n w_{i_j}^2\right)^{1/2}\right) \\
&=& a^2F(w,t).
\end{eqnarray*}
\end{proof}

\begin{corollary}
\label{cor:l1}
There exists a universal constant $c_4 > 0$ such that the
following holds: Let $A$ be a $k \times n$ matrix with entries independent and uniform
random amongst $\{-1,+1\}$. For any vector $w \in \R^n$ and any
integer $t \geq 4$, we have:
$$
\Pr\left[\|Aw\|_1 > c_4 k F(w,t) \right] \leq e^{-t^2 k/4}.
$$
\end{corollary}

\begin{proof}
Let $t \geq 4$ be integer. For each way of choosing $k$ non-negative integers $z_1,\dots,z_k$
such that $\sum_i z_i = k$, define an event $E_{z_1,\dots,z_k}$ that
happens when:
$$
\forall h : |(Aw)_h| \geq c_1 z_h F(w,t).
$$
We wish to bound $\Pr[E_{z_1,\dots,z_k}]$. Using Lemma~\ref{lem:Fineq}
we get:
\begin{eqnarray*}
|(Aw)_h| \geq c_1 z_h F(w,t) &\Rightarrow& \\
|(Aw)_h| \geq c_1F(w, \sqrt{z_h} t),\\
\end{eqnarray*}
whenever $z_h \geq 1$. Using that the distribution of $(Aw)_h$ is
symmetric around $0$, we get from Theorem~\ref{thm:tailbounds} that
$$
\Pr[E_{z_1,\dots,z_k}] \leq \prod_{h=1}^k 2e^{-(\sqrt{z_h} t)^2/2} = 2^ke^{-k t^2/2}.
$$
Let $c_4 = 2c_1$ where $c_1$ is the constant from
Theorem~\ref{thm:tailbounds}. If 
$$
\|Aw\|_1 > c_4 k F(w,t)
$$
then at least one of the events $E_{z_1,\dots,z_k}$ must happen. Since
there are $\binom{k + k-1}{k} \leq 2^{2k}$ such events, we conclude
from a union bound that for integer $t \geq 4$, we have
$$
\Pr\left[\|Aw\|_1 > c_4 k F(w,t)\right] \leq 2^{3k}e^{-k t^2/2} <
e^{-k(t^2/2 - 3)} \leq e^{-k(t^2/2 - t^2/4)} = e^{-kt^2/4}.
$$
\end{proof}

We will also need the Chernoff bound:
\begin{theorem}[Chernoff Bound]
\label{thm:chernoff}
Let $X_1,\dots,X_n$ be independent $\{0,1\}$-random variables and let
$X = \sum_i X_i$. Then for any $0 < \delta < 1$:
$$
\Pr[X \leq (1-\delta)\E[X]] \leq e^{-\delta^2 \E[X]/2}
$$
\end{theorem}

\subsection{The Net Argument}
We are ready to prove Lemma~\ref{lem:smallcoord}. As highlighted
above, our proof will use a net argument. What we wanted to construct,
was a collection of vectors $W$, such that any vector $w$ with
$\|w\|_0 \leq T$ and $\|w\|_1 = 1$ was close to some vector in $W$,
and moreover, for a random matrix $A \in \{-1,+1\}^{n \times n}$,
there is a non-zero probability that every vector $w \in W$ satisfies
that there are many coordinates in $Aw$ that are less than
$-a_1(\sqrt{\lg(n/T)/T})$, but few that are greater than
$a_2(\sqrt{\lg(n/T)/T})$ in absolute value for some constants $a_1 <
a_2$. We formulate this property in the following definition:
\begin{definition}
Let $A \in \R^{n \times n}$ be an $n \times n$ real matrix and let $w
\in \R^n$. We say that $A$ is $(t,b,T)$-\emph{typical} for $w$ if the following
two holds:
\begin{itemize}
\item $\left| \left\{ i : (Aw)_i < -c_2 F(w,t) \right\} \right| \geq
  \ln \binom{n}{T}$, where $c_2$ is the constant from Theorem~\ref{thm:tailbounds}.
\item For all sets of $k = \ln \binom{n}{T}$ rows of $A$, it
  holds that the corresponding $k \times n$ submatrix $\bar{A}$
  satisfies $\|\bar{A} w\|_1 < b k  F(w,t)$.
\end{itemize}
\end{definition}
Observe that the second property in the definition says that there
cannot be too many very large coordinates in $Aw$. The following lemma
shows that for a fixed vector $w \in \R^n$ and a random matrix $A$,
there is a very high probability that $A$ is typical for $w$:

\begin{lemma}
\label{lem:typical}
There exists universal constants $b_1,b_2 > 0$ such that the following
holds for all $n$ sufficiently large: Let $A$ be a random $n \times n$ matrix with each entry chosen independently
and uniformly at random amongst $\{-1,+1\}$. Define 
$$
\tau(T) := \left\lfloor \sqrt{c_3^{-1} \ln\left( \frac {n}{12 c_3\ln \binom{n}{T}}
  \right)} \right\rfloor
$$
where $c_3$ is the constant from Theorem~\ref{thm:tailbounds}. Then
the following holds for all $T \leq n/b_1$: If $w \in \R^n$ with $\|w\|_0 \leq
T$, then:
$$
\Pr\left[A \textrm{ is } (\tau(m),b_2,T)\textrm{-typical for } w\right] \geq
1-\binom{n}{T}^{-3}.
$$
Furthermore, it holds that $\tau(T) = \Omega(\sqrt{\ln(n/T)})$.
\end{lemma}

\begin{proof}
Let $a_i$ be the $i$'th row of $A$. Using that the distribution of
$\langle a ,w \rangle$ is symmetric around $0$,
Theorem~\ref{thm:tailbounds} gives us:
$$
\forall t > 0 : \Pr\left[\langle w, x\rangle < -c_2 F(w,t) \right] \geq c_3^{-1} e^{-c_3 t^2}
$$
Defining 
$$
\tau(T) := \left\lfloor \sqrt{c_3^{-1} \ln\left( \frac {n}{12 c_3\ln \binom{n}{T}}
  \right)} \right\rfloor
$$
gives
$$
\Pr\left[\langle w, x\rangle < -c_2 F(w,\tau(T)) \right] \geq
c_3^{-1} e^{-c_3 \tau(T)^2} \geq \frac{12 \ln \binom{n}{T}}{n}.
$$
Now define $X_i$ to take the value $1$ if $\langle w, x\rangle < -c_2
F(w,\tau(T))$ and $0$ otherwise. Let $X = \sum_i X_i$. Then $\E[X] \geq 12
\ln\binom{n}{T}$. It follows from the Chernoff bound
(Theorem~\ref{thm:chernoff}) that
$$
\Pr\left[X \leq \ln\binom{n}{T}\right] \leq e^{-(11/12)^2 6 \ln
  \binom{n}{T}} = \binom{n}{T}^{-11^2/24} <
\binom{n}{T}^{-5} < \frac{\binom{n}{T}^{-3}}{2}.
$$
Thus the first requirement for being $(\tau(T),b_2,T)$-typical is
satisfied with probability at least 
$$
1-\frac{\binom{n}{T}^{-3}}{2}.
$$
Our next step is to show the same for the second property. To do
this, we will be using Corollary~\ref{cor:l1}, which requires $t \geq
4$. If we restrict $T \leq n/b_1$ for a big enough constant $b_1
> 0$,
this is indeed the case for our choice of $\tau(T)$ (since we can make $\ln
\binom{n}{T} \leq n/b'$ for any desirable constant $b'>0$ by
setting $b_1$ to a large enough constant). Thus from here on, we fix $b_1$ to a
constant large enough that $\tau(T) \geq 4$ whenever $T \leq n/b_1$.

Now fix a set of $k = \ln \binom{n}{T}$ rows of
$A$ and let $\bar{A}$ denote the corresponding $k \times n$
submatrix. Corollary~\ref{cor:l1} gives us that $\Pr[\|\bar{A}w\|_1 > c_4 k
F(w,\alpha \tau(T))] \leq e^{-\alpha^2 \tau(T)^2k/4} =
\binom{n}{T}^{-\alpha^2\tau(T)^2/4}$ for any constant $\alpha \geq
1$. If we set $\alpha$ to a large enough constant, we get that
$\alpha^2 \tau(T)^2/4 \geq 5 \ln(en/k)$. Thus $\Pr[\|\bar{A}w\|_1 > c_4 k
F(w,\alpha \tau(T))] \leq e^{-5k \ln(en/k)}$. There are $\binom{n}{k}
\leq e^{k \ln(en/k)}$ sets of $k$ rows in $A$, thus by a union bound,
we get that with probability at least $1-e^{-4k \ln(en/k)} \geq
1-\binom{n}{T}^{-3}/2$, it simultanously holds that $\|\bar{A}w\|_1 \leq c_4 k
F(w,\alpha \tau(T))$ for all $k \times n$ submatrices $\bar{A}$ of
$A$. From Lemma~\ref{lem:Fineq}, this implies $\|\bar{A}w\|_1 \leq c_4\alpha^2 k
F(w,\tau(T))$. Another union bound shows that $A$ is $(\tau(T), b_2,
T)$-typical for $w$ with probability at least $1-\binom{n}{T}^{-3}$ if
we set $b_2 = c_4 \alpha^2$.

What remains is to show that $\tau(T) = \Omega(\sqrt{\ln(n/T)})$. We
see that:
\begin{eqnarray*}
\left\lfloor \sqrt{c_3^{-1} \ln\left( \frac {n}{12 c_3\ln \binom{n}{T}}
  \right)} \right\rfloor &\geq& 
\left\lfloor \sqrt{c_3^{-1} \ln\left( \frac {n}{12 c_3 T \ln(en/T)}
  \right)} \right\rfloor
\end{eqnarray*}
If $b_1$ is a large enough constant (i.e. $T$ small enough compared to $n$), then
$\ln(en/T) \leq 2\ln(n/T)$. Thus for large enough $b_1$:
\begin{eqnarray*}
\tau(T) &\geq& \left\lfloor \sqrt{c_3^{-1} (\ln(n/T) -\ln \ln(n/T) -
    \ln(24 c_3))} \right\rfloor
\end{eqnarray*}
 the terms $\ln(\ln(n/T))$ and $\ln(24 c_3)$ are at most small constant
factors of $\ln(n/T)$ and the claim follows for $b_1$ big enough.
\end{proof}

We are ready to define the net $W$ which we will union bound over and
show that a random matrix is typical for all vectors in $W$
simultanously with non-zero probability. Let $b_1,b_2$ be the constants from
Lemma~\ref{lem:typical} and let $T \leq n/b_1b^2_3$ with $b_3 > b_2$ a large constant. Construct a collection $W$ of vectors in
$\R^n$ as
follows: For every choice of $m$ distinct coordinates $i_1,\dots,i_T \in [n]$, and every choice
of $T$ integers $z_1,\dots,z_T$ with $\sum_j |z_j| \leq (b_3+1)T$, add the
vector $w$ with coordinates $w_{i_j} = z_j/(b_3T)$ for $j=1,\dots,T$ and
$w_h = 0$ for $h \notin \{i_1,\dots,i_T\}$ to $W$. We have
$$
|W| \leq \binom{n}{T} 2^T \sum_{i=0}^{(b_3+1)T} \binom{T+i-1}{T}.
$$
To see this, observe that $\binom{n}{T}$ counts the number of ways to
choose $i_1,\dots,i_T$, $2^T$ counts the number of ways of choosing
the signs of the $z_j$'s and $\binom{T+i-1}{T}$ counts the number of
ways to choose $T$ non-negative integers summing to $i$.

If $b_3$ is large enough, then $(b_3+1)T \leq (b_3+1)n/(b_1 b^2_3)$ is much smaller than $n$ and
hence $|W| < \binom{n}{T}^3$. Since $T < n/b_1$, we can use Lemma~\ref{lem:typical} and a union
bound over all $w \in W$ to conclude that there exists an $n \times n$
matrix $A$ with coordinates amongst $\{-1,+1\}$, such that $A$ is
$(\tau(T), b_2, T)$-typical for all $w \in W$ with $\tau(T) =
\Omega(\sqrt{\ln(n/T)})$. We claim that this implies the following:

\begin{lemma}
\label{lem:typImplies}
Let $A \in \{-1,+1\}^{n \times n}$ be $(\tau(T),b_2,T)$-typical for
all $w \in W$. If $\ln n \leq T$, then for any vector $x$ with $\|x\|_0 \leq T$ and
$\|x\|_1=1$, there exists an index $i$ such that 
$$
(Ax)_i \leq -\Omega\left( \sqrt{\frac{\ln(n/T)}{T}} \right).
$$
\end{lemma}

\begin{proof}
Let $x$ be as in the lemma, i.e. $\|x\|_0 \leq T$ and $\|x\|_1=1$. To prove the lemma, we will ``round''
the arbitrary vector $x$ to a close-by vector in $W$, allowing us to
use that $A$ is typical for all vectors in $W$. 

For a vector $x$ with $\|x\|_0 \leq T$ and $\|x\|_1=1$, let $f(x)
\in W$ be the vector in $W$ obtained by rounding each coordinate of
$x$ up (in absolute value) to the nearest multiple of $1/(b_3
T)$. The vector $f(x)$ is in $W$ since $\sum_i |f(x)_i| \cdot (b_3 T)
\leq (b_3 T)(1 + T/b_3T) \leq 
b_3T + T = (b_3+1)T$ and that each coordinate is an integer
multiple of $1/(b_3T)$. 

Notice that the vector $\tilde{x} = x - f(x)$ has
at most $T$ non-zero coordinates, all between
$0$ and $1/(b_3 T)$ in absolute value. Now observe that:
\begin{eqnarray*}
Ax &=& A(f(x)+\tilde{x}) \\
&=& Af(x) + A\tilde{x}.
\end{eqnarray*}
Since $f(x) \in W$, we have that $A$ is $(\tau(T),b_2,T)$-typical
for $f(x)$. Hence there exists at least $\ln \binom{n}{T}$
coordinates of $Af(x)$ such that $(Af(x))_i < -c_2
F(f(x),\tau(T))$. If we let $i_j$ denote the index of the $j$'th
largest coordinate of $f(x)$ in terms of absolute value, then we have:
\begin{eqnarray*}
F(f(x),\tau(T)) &=& \left(\sum_{j=1}^{\lfloor \tau(T)^2 \rfloor}
    |f(x)_{i_j}| + \tau(T) \left( \sum_{j = \lfloor \tau(T)^2
        \rfloor  + 1 }^n f(x)_{i_j}^2\right)^{1/2}\right).
\end{eqnarray*}
By construction, we have that $\|f(x)\|_1 \geq \|x\|_1 = 1$ ($f$
rounds coordinates up in absolute value). This means that either $\sum_{j=1}^{\lfloor \tau(T)^2 \rfloor}
    |f(x)_{i_j}| \geq 1/2$ or $\sum_{j = \lfloor \tau(T)^2
        \rfloor  + 1 }^n |f(x)_{i_j}| \geq 1/2$. In the first case,
      we have $F(f(x),\tau(T)) \geq 1/2$. In the latter case, we
      use Cauchy-Schwartz to conclude that
$$
 \left( \sum_{j = \lfloor \tau(T)^2
        \rfloor  + 1 }^n f(x)_{i_j}^2\right)^{1/2} \geq
    (1/2)/\|f(x)\|^{1/2}_0 \geq 1/(2\sqrt{T}).
$$
To see how this follows from Cauchy-Schwartz, define $\overline{f(x)}$ as
the vector having the same coordinates as $f(x)$ for indices $i_j$
with $j \geq \lfloor \tau(T)^2 \rfloor+1$ and $0$ elsewhere. Define $y$ as the vector
with a $1$ in all coordinates $i$ where $\overline{f(x)}_i \neq 0$ and $0$
elsewhere. Then by Cauchy-Schwartz
$$
1/2 \leq \sum_i |y_i \overline{f(x)}_i| \leq \sqrt{\|\overline{f(x)}\|_2 \|y\|_2} \leq \sqrt{\|\overline{f(x)}\|_2 \|f(x)\|_0}.
$$
Hence we must have 
$$
F(f(x),\tau(T)) \geq \min\left\{1/2, 
      \frac{\tau(T)}{\sqrt{T}} \right\}.
$$
Therefore, we have at least $\ln \binom{n}{T}$ coordinates $i$ of $A
f(x)$ with
$$
(Af(x))_i \leq - c_2 \min\left\{1/2, 
      \frac{\tau(T)}{\sqrt{T}} \right\},
$$
where $c_2$ is the constant from Theorem~\ref{thm:tailbounds}.

We let $T$ denote an arbitrary subset of $\ln \binom{n}{T}$ such indices. What remains is to show
that $A\tilde{x}$ cannot cancel out these very negative
coordinates. To prove this, we will write $\tilde{x}$ as a sum
$\|\tilde{x}\|_1 \cdot \sum_{j=0}^{\infty}
\alpha_j \tilde{x}^{(j)}$ with each $\tilde{x}^{(j)} \in W$. The idea
is to repeatedly apply the function $f$ to find a vector in $W$ that
is close to $\tilde{x}$. We then append that to the sum and recurse on
the ``rounding error'' $\tilde{x}-f(\tilde{x})$. To do this more
formally, define
\begin{eqnarray*}
\hat{x}^{(0)} &:=& \tilde{x}/\|\tilde{x}\|_1\\
\tilde{x}^{(0)} &:=& f(\hat{x}^{(0)})\\
\alpha_0 &:=& 1
\end{eqnarray*}
and for $j>0$, define
\begin{eqnarray*}
\hat{x}^{(j)} &:=& \hat{x}^{(j-1)} - \alpha_{j-1}\tilde{x}^{(j-1)} \\
\tilde{x}^{(j)} &:=& f(\hat{x}^{(j)}/\|\hat{x}^{(j)}\|_1)\\
\alpha_j &:=& \|\hat{x}^{(j)}\|_1.
\end{eqnarray*}
We would like to show that $\tilde{x} = \|\tilde{x}\|_1 \cdot \sum_{j=0}^{\infty}
\alpha_j \tilde{x}^{(j)}$. The first step in proving this, is to
bound $\alpha_j$ and $\hat{x}^{(j)}$ as $j$ tends to infinity. We see that:
$\alpha_j =
\|\hat{x}^{(j)}\|_1$. For $j \neq 0$, we therefore get:
\begin{eqnarray*}
\alpha_j &=& \|\hat{x}^{(j-1)} - \alpha_{j-1} \tilde{x}^{(j-1)} \|_1
\\
&=&  \|\hat{x}^{(j-1)} - \alpha_{j-1}
f(\hat{x}^{(j-1)}/\alpha_{j-1})\|_1 \\
&=& \alpha_{j-1} \|\hat{x}^{(j-1)}/\alpha_{j-1} - 
f(\hat{x}^{(j-1)}/\alpha_{j-1})\|_1\\
&\leq& \alpha_{j-1} T (1/b_3 T)\\
&=& \alpha_{j-1}/b_3.
\end{eqnarray*}
This implies that $\lim_{j \to \infty} \alpha_j = 0$. Since $\alpha_j
= \|\hat{x}^{(j)}\|_1$, this also implies that $\lim_{j \to \infty}
\hat{x}^{(j)} = 0$. We can now conclude that:
\begin{eqnarray*}
\|\tilde{x}\|_1 \cdot \sum_{j=0}^{\infty}
\alpha_j \tilde{x}^{(j)} &=& \\
\|\tilde{x}\|_1 \cdot \left(f(\hat{x}^{(0)})+ \sum_{j=1}^{\infty}
\alpha_j 
\left(\frac{\hat{x}^{(j)}-\hat{x}^{(j+1)}}{\alpha_j}\right) \right) &=& \\
\|\tilde{x}\|_1 \cdot \left(f(\hat{x}^{(0)}) + \sum_{j=1}^{\infty}
\hat{x}^{(j)}-\hat{x}^{(j+1)} \right) &=& \\
\|\tilde{x}\|_1 \cdot \left(f(\hat{x}^{(0)}) +\left(\hat{x}^{(1)} -
    \lim_{j \to \infty} \hat{x}^{(j)}\right) \right)
&=& \\
\|\tilde{x}\|_1 \cdot \left(f(\hat{x}^{(0)}) +\left( \hat{x}^{(0)} -
    \tilde{x}^{(0)} \right)\right) &=& \\
\|\tilde{x}\|_1 \cdot \left(f(\hat{x}^{(0)}) +\left( \hat{x}^{(0)} -
   f(\hat{x}^{(0)})\right)\right) &=& \\
\tilde{x}.
\end{eqnarray*}
Thus we have that $A\tilde{x} = A\left(\|\tilde{x}\|_1
  \sum_{j=0}^\infty \alpha_j \tilde{x}^{(j)}\right)$. Now let
$\bar{A}$ be the $\ln \binom{n}{T} \times n$ submatrix of $A$
corresponding to the rows in the set $T$. We wish to bound $\|\bar{A}
\tilde{x}\|_1$. Using the triangle inequality and that $A$ is
$(\tau(T),b_2,T)$-typical for all vectors in $W$, we get:
$$
\|\bar{A} \tilde{x}\|_1 \leq \|\tilde{x}\|_1 \sum_{j=0}^\infty \alpha_j \|\bar{A}
\tilde{x}^{(j)}\|_1 \leq \|\tilde{x}\|_1 \ln \binom{n}{T} b_2 \sum_{j=0}^\infty \alpha_j F(\tilde{x}^{(j)}, \tau(T)).
$$
We want to argue that $F(\tilde{x}^{(j)}, \tau(T))$ is small. For
this, we will bound the magnitude of coordinates in
$\tilde{x}^{(j)}$. We start with $j\neq 0$, in which case we have
 $\tilde{x}^{(j)} =
f(\hat{x}^{(j)}/\|\hat{x}^{(j)}\|_1)$. By definition of $f$, this means
that $\|\tilde{x}^{(j)}\|_\infty \leq
\|\hat{x}^{(j)}/\|\hat{x}^{(j)}\|_1\|_\infty + 1/b_3 T$. Using that
\begin{eqnarray*}
\hat{x}^{(j)} &=& \hat{x}^{(j-1)} - \alpha_{j-1}\tilde{x}^{(j-1)} \\
&=& \hat{x}^{(j-1)} -
\alpha_{j-1}f(\hat{x}^{(j-1)}/\|\hat{x}^{(j-1)}\|_1) \\
&=& \hat{x}^{(j-1)} -
\|\hat{x}^{(j-1)}\|_1 f(\hat{x}^{(j-1)}/\|\hat{x}^{(j-1)}\|_1) \\
&=& \|\hat{x}^{(j-1)}\|_1 \left(\hat{x}^{(j-1)}/\|\hat{x}^{(j-1)}\|_1
  - f(\hat{x}^{(j-1)}/\|\hat{x}^{(j-1)}\|_1) \right),
\end{eqnarray*}
we get by definition of $f$ that $\left(\hat{x}^{(j-1)}/\|\hat{x}^{(j-1)}\|_1
  - f(\hat{x}^{(j-1)}/\|\hat{x}^{(j-1)}\|_1) \right)$ has at most $T$
non-zero coordinates, all between $0$ and $1/b_3 T$ in absolute
value. We have thus shown that 
$$
\|\tilde{x}^{(j)}\|_\infty \leq
\|\hat{x}^{(j)}/\|\hat{x}^{(j)}\|_1\|_\infty + 1/b_3 T \leq
\left(1+\frac{\|\hat{x}^{(j-1)}\|_1}{\|\hat{x}^{(j)} \|_1} \right)
\cdot \frac{1}{b_3 T} = \left(1+\frac{\alpha_{j-1}}{\alpha_j} \right)
\cdot \frac{1}{b_3 T}.
$$
We therefore have for $j\neq 0$ that:
\begin{eqnarray*}
F(\tilde{x}^{(j)}, \tau(m)) 
&\leq& \\
\left(\sum_{j=1}^{\lfloor \tau(T)^2 \rfloor}
    \left(1+\frac{\alpha_{j-1}}{\alpha_j} \right)
\cdot \frac{1}{b_3 T} + \tau(T) \left( \sum_{j = \lfloor \tau(T)^2
        \rfloor  + 1 }^T \left(\left(1+\frac{\alpha_{j-1}}{\alpha_j} \right)
\cdot \frac{1}{b_3 T}\right)^2\right)^{1/2}\right)
&\leq& \\
\tau(T)^2\left(1+\frac{\alpha_{j-1}}{\alpha_j} \right)
\cdot \frac{1}{b_3 T} + \tau(T) \sqrt{T}\left(1+\frac{\alpha_{j-1}}{\alpha_j} \right)
\cdot \frac{1}{b_3 T}
&=& \\
\tau(T)^2\left(1+\frac{\alpha_{j-1}}{\alpha_j} \right)
\cdot \frac{1}{b_3 T} + \tau(T)\left(1+\frac{\alpha_{j-1}}{\alpha_j} \right)
\cdot \frac{1}{b_3 \sqrt{T}}.
\end{eqnarray*}
For $j=0$, we see that 
$$
\|\tilde{x}^{(0)}\|_\infty = \|f(\hat{x}^{(0)}\|_\infty \leq
\|\tilde{x}/\|\tilde{x}\|_1\|_\infty + 1/b_3 T = \|\tilde{x}\|_\infty /\|\tilde{x}\|_1+ 1/b_3 T.
$$
But $\tilde{x}$ was the difference between $x$ and $f(x)$ and
therefore $\|\tilde{x}\|_\infty \leq 1/b_3 T$ and we get:
$$
\|\tilde{x}^{(0)}\|_\infty \leq \left(1 +
  \frac{1}{\|\tilde{x}\|_1}\right)/b_3 T.
$$
We then have:
\begin{eqnarray*}
F(\tilde{x}^{(0)}, \tau(T)) &\leq& \tau(T)^2\left(1+\frac{1}{\|\tilde{x}\|_1} \right)
\cdot \frac{1}{b_3 T} + \tau(T)\left(1+\frac{1}{\|\tilde{x}\|_1} \right)
\cdot \frac{1}{b_3 \sqrt{T}},
\end{eqnarray*}
which allows us to conclude:
\begin{eqnarray*}
\|\bar{A} \tilde{x}\|_1 &\leq&  \|\tilde{x}\|_1 \ln \binom{n}{T} b_2
\sum_{j=0}^\infty \alpha_j F(\tilde{x}^{(j)}, \tau(T)) \\
&=& \|\tilde{x}\|_1 \ln \binom{n}{T} b_2 \cdot \left(\frac{\tau(T)^2}{b_3 T} +
    \frac{\tau(T)}{b_3 \sqrt{T}} \right) \cdot
  \left(\left(1+\frac{1}{\|\tilde{x}\|_1} \right) +
    \sum_{j=1}^\infty\left(\alpha_{j-1} +
      \alpha_j\right)\right)\\
&=& \|\tilde{x}\|_1 \ln \binom{n}{T} b_2 \cdot \left(\frac{\tau(T)^2}{b_3 T} +
    \frac{\tau(T)}{b_3 \sqrt{T}} \right) \cdot
  \left(\frac{1}{\|\tilde{x}\|_1}+
    2 \cdot \sum_{j=0}^\infty \alpha_j \right). 
\end{eqnarray*}
Recalling that $\alpha_j \leq \alpha_{j-1}/b_3$, we can write: $\sum_{j=0}^\infty \alpha_j \leq \sum_{j=0}^\infty (1/b_3)^j \leq
2$ for $b_3 \geq 2$. Using that $\tilde{x}$ has at most $T$ non-zero
coordinates, all of absolute value at most $1/b_3 T$, we also have
$\|\tilde{x}\|_1 \leq 1/b_3$. Therefore we conclude:
\begin{eqnarray*}
\|\bar{A} \tilde{x}\|_1 &\leq& \ln \binom{n}{T} b_2 \left(\frac{\tau(T)^2}{b_3 T} +
    \frac{\tau(T)}{b_3 \sqrt{T}} \right) \cdot
  \left(1+
    4/b_3 \right). 
\end{eqnarray*}
Hence there must be one of the coordinates $i$ where
$$
|(\bar{A}\tilde{x})_i| \leq \frac{b_2}{b_3} \left(\frac{\tau(T)^2}{ T} +
    \frac{\tau(T)}{ \sqrt{T}} \right) \cdot
  \left(1+
    4/b_3 \right). 
$$
But all $\ln \binom{n}{T}$ coordinates of $\bar{A}f(x)$ were less than
$$
-c_2 \min\left\{1/2, 
      \frac{\tau(T)}{\sqrt{T}} \right\}
$$
thus for $b_3$ a large enough constant and $\tau(T) \leq \sqrt{T}$, there must be a coordinate $i$
where
$$
(Ax)_i \leq - \Omega\left( \min\left\{1,\sqrt{\frac{\ln(n/T)}{T}}
  \right\} \right).
$$
Under the assumption that $\tau(T) \leq \sqrt{T} \Leftrightarrow
\ln(n/T) \leq T \Leftarrow \ln n \leq T$, this simplifies to:
$$
(Ax)_i \leq - \Omega\left( \sqrt{\frac{\ln(n/T)}{T}}
   \right).
$$
\end{proof}

We have argued that the random matrix $A$ was typical for all
vectors in $W$ simultanously with non-zero probability. Combining this
with Lemma~\ref{lem:typImplies} finally concludes the proof of Lemma~\ref{lem:smallcoord}.
g

%% file: experiments.tex

\section{Experiments}
\label{sec:experiments}
Gradient Boosting \cite{gb1,gb2} is probably the most popular boosting algorithm in practice. It has several highly efficient open-source implementations \cite{xgboost,lightgbm, catboost} and obtain state-of-the-art performance in many machine learning tasks \cite{lightgbm}. In this section we demonstrate how our sparsification algorithm can be combined with Gradient Boosting. For simplicity we consider a single dataset in this section, the Flight Delay dataset \cite{airline}, see \Cref{sec:appendix_experimental_details} for similar results on other dataset. 

We train a classifier with $T=500$ hypotheses using LightGBM \cite{lightgbm} which we sparsify using \Cref{lemma:margin_preservation} to have $T'=80$ hypotheses. The sparsified classifier is guaranteed to preserve all margins of the original classifier to an additive $O(\sqrt{\log(n/T')/T'})$. The cumulative margins of the sparsified classifier and the original classifier are depicted in \Cref{fig:margins}. Furthermore, we also depict the cumulative margins of a LightGBM classifier trained to have $T'=80$ hypotheses. 
\begin{figure}[t]
\vskip 0.2in
\begin{center}
\centerline{\includegraphics[width=0.8\columnwidth]{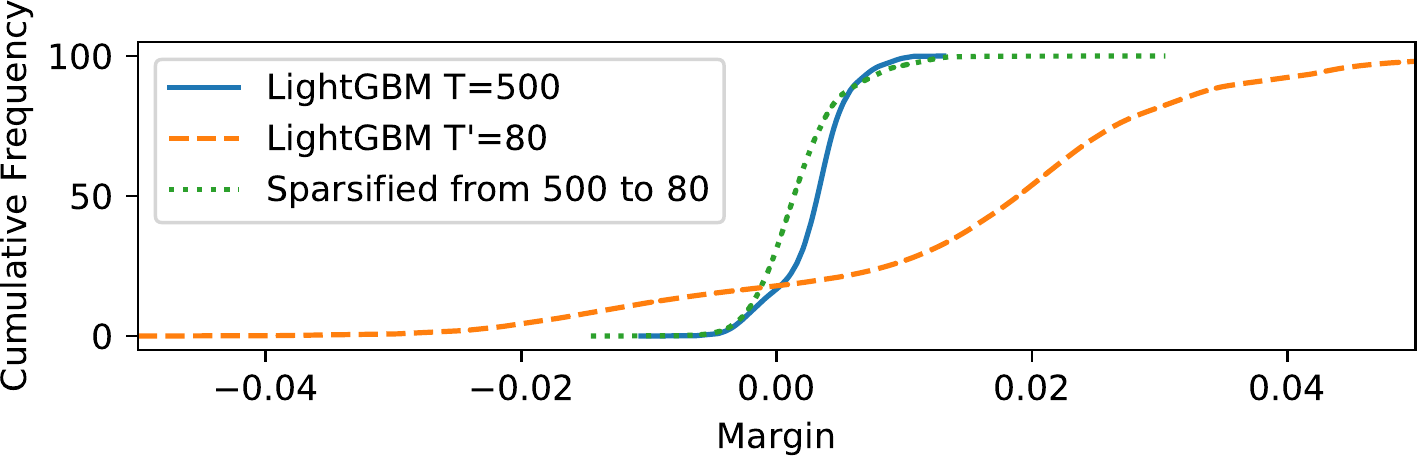}}
\caption{ The plot depicts the cumulative margins of three classifiers: (1) a LightGBM classifier with $500$ hypotheses (2) a classifier sparsified from $500$ to $80$ hypotheses and (3) a LightGBM classifier with $80$ hypotheses.  }
\label{fig:margins}
\end{center}
\vskip -0.2in
\end{figure}
First observe the difference between the LightGBM classifiers with $T=500$ and $T'=80$ hypotheses (blue and orange in \Cref{fig:margins}). The margins of the classifier with $T=500$ hypotheses vary less. It has fewer points with a large margin, but also fewer points with a small margin. The margin distribution of the sparsified classifier with $T'=80$ approximates the margin distribution of the LightGBM classifier with $T=500$ hypotheses. Inspired by margin theory one might suspect this leads to better generalization. To investigate this, we performed additional experiments computing AUC and classification accuracy of several sparsified classifiers and LightGBM classifiers on a test set (we show the results for multiple sparsified classifiers due to the randomization in the discrepancy minimization algorithms). The experiments indeed show that the sparsified classifiers outperform the LightGBM classifiers with the same number of hypotheses. See \Cref{fig:scores} for test AUC and test classification accuracy. 

\begin{figure}[t]
\vskip 0.2in
\begin{center}
\includegraphics[width=0.48\columnwidth]{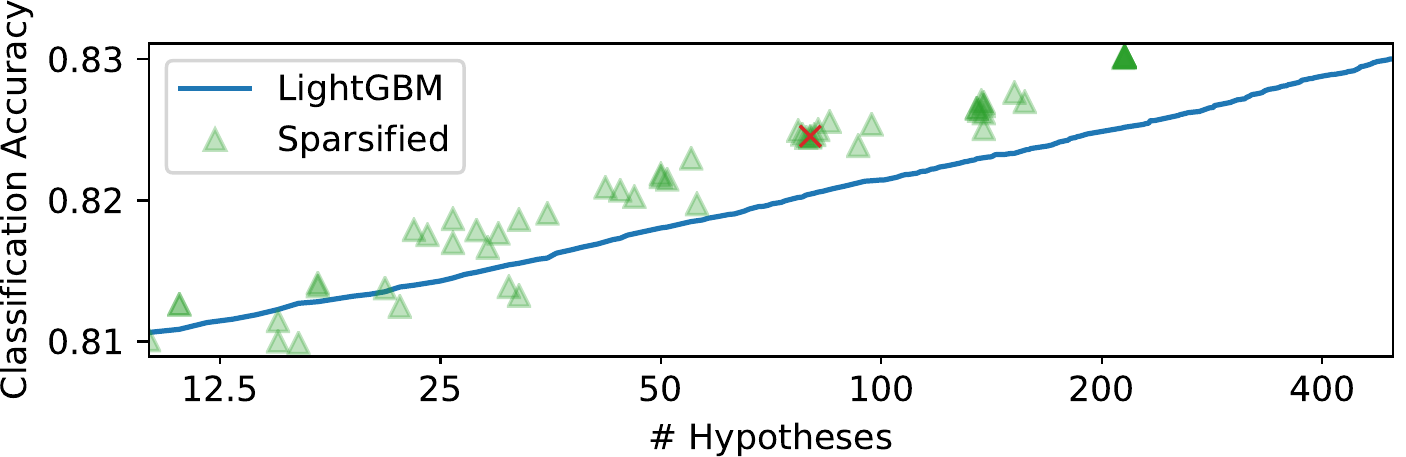}
\includegraphics[width=0.48\columnwidth]{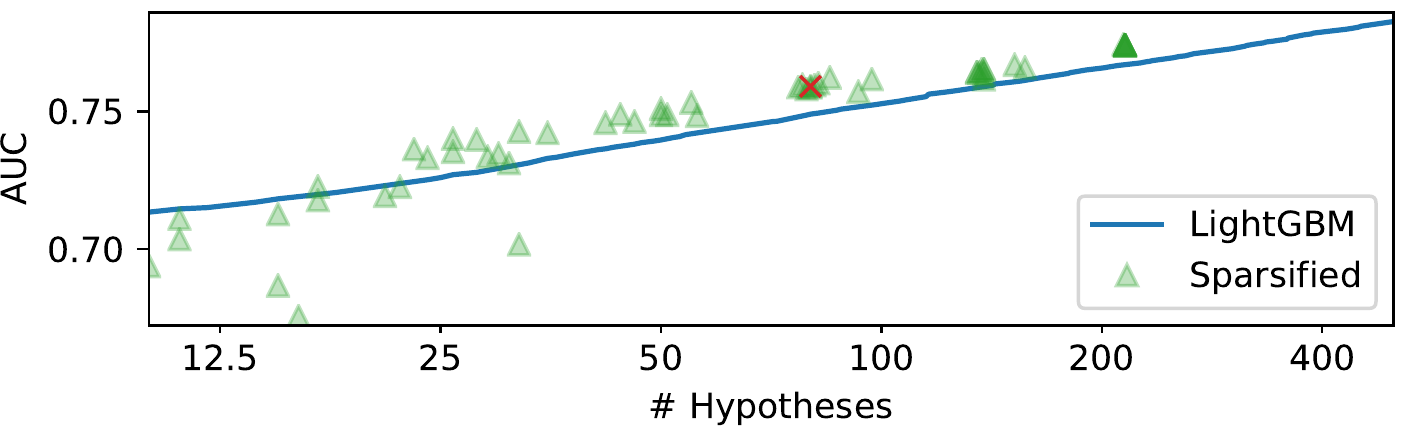}
\caption{The plot depicts test AUC and test classification accuracy of a LightGBM classifier during training as the number of hypotheses increase (in blue). Notice the x-axis is logarithmically scaled. The final classifier with $500$ hypotheses was sparsified with \Cref{lemma:margin_preservation} multiple times to have between $T/2$ to $T/16$ hypotheses. The green triangles show test AUC and test accuracy of the resulting sparsified classifiers. The red cross represents the sparsified classifier used to plot the cumulative margins in \Cref{fig:margins}.}
\label{fig:scores}
\end{center}
\vskip -0.2in
\end{figure}

\paragraph{Further Experiments and Importance Sampling.}
Inspired by the experiments in \cite{emargin, lightgbm, xgboost} we also performed the above experiments on the Higgs \cite{higgs} and Letter \cite{letter} datasets. See \Cref{sec:appendix_experimental_details} for (1) further experimental details and (2) for cumulative margin, AUC and test accuracy plots on all dataset for different values of $n$ and $T$. 

As mentioned in Section~\ref{sec:upper}, one could use importance sampling for sparsification. It has a slightly worse theoretical guarantee, but might work better in practice. \Cref{sec:appendix_experimental_details} also contains test AUC and test accuracy of the classifiers that result from using importance sampling instead of our algorithm based on discrepancy minimization. Our algorithm and importance sampling are both random so the experiments were repeated several times. On average over the experiments, our algorithm obtains a better test AUC and classification accuracy than importance sampling. 
\\\\
A Python/NumPy implementation of our sparsification algorithm (Theorem~\ref{lemma:margin_preservation}) can be found at: 
\begin{center}
\url{https://github.com/AlgoAU/DiscMin}
\end{center}



%% file: conclusion.tex
\section{Conclusion}
A long line of research into obtaining a large minimal margin using few hypotheses \cite{breiman1999prediction, grove1998boosting, bennett2000column, ratsch2002maximizing} culminated with the AdaBoostV \cite{ratsch2005efficient} algorithm. AdaBoostV was later conjectured by \cite{nie2013open} to provide an optimal trade-off between minimal margin and number of hypotheses. In this article, we introduced SparsiBoost which refutes the conjecture of \cite{nie2013open}. Furthermore, we show a matching lower bound, which implies that SparsiBoost is optimal. 

The key idea behind SparsiBoost, is a sparsification algorithm that reduces the number of hypotheses while approximately preserving the entire margin distribution. Experimentally, we combine our sparsification algorithm with LightGBM. We find that the sparsified classifiers obtains a better margin distribution, which typically yields a better test AUC and test classification error when compared to a classifier trained directly to the same number of hypotheses. 


%% file: appendix.tex
\section{Rewriting bounds}
\label{sec:appen}
In the following, $\lg$ is the base $2$ logarithm.
Here we show that if $T \leq c_1 \lg(nv^2)/v^2$ for a constant $c_1
> 0$ and $T \leq n/2$, then $v \leq c_2 \sqrt{\lg(n/T)/T}$ for a
constant $c_2 >0$. We split the proof in two parts:
\begin{enumerate}
\item If $v^2 >
1/n$: First note that $T \leq c_1 \lg(nv^2)/v^2
\Rightarrow v \leq \sqrt{c_1 \lg(nv^2)/T}$. We now claim that
$\lg(nv^2) \leq c_3 \lg(n/T)$ for a constant $c_3 > 0$. Using our
bound $v \leq \sqrt{c_1 \lg(nv^2)/T}$ we get $\lg(nv^2) \leq \lg(n
c_1 \lg(nv^2)/T) = \lg(n/T) + \lg c_1 + \lg \lg (nv^2)$. We have $\lg \lg(nv^2) \leq \lg(nv^2)*0.9$ (since $\lg \lg x <
(\lg x)*0.9$ for all $x > 1$). This implies $\lg(nv^2)/10 \leq \lg(n/T) +
\lg c_1$. Since $T \leq n/2$ we have $\lg(n/T) \geq 1$ and it follows
that $\lg(nv^2) \leq (10 + \max\{0,10\lg c_1\}) \lg(n/T)$. Defining $c_3
= (10 + \max\{0,10\lg c_1\})$ we thus conclude that $v \leq \sqrt{c_1
  \lg(nv^2)/T} \Rightarrow v \leq \sqrt{c_1 c_3 \lg(n/T)/T}$ as claimed.
\item If $v^2 \leq 1/n$: Here we immediately have $v \leq
  \sqrt{\lg(n/T)/T}$ since $T \leq n/2$.
\end{enumerate}

Next we show that if $v \geq c_1 \sqrt{\lg(n/T)/T}$ for a constant
$c_1 > 0$ and $v$ satisfies $v  > \sqrt{1/n}$, then $T \geq c_2 \lg(nv^2)/v^2$ for a constant $c_2 > 0$. We
split the proof in two parts:
\begin{enumerate}
\item If $T \leq \lg(nv^2)/v^2$: Observe that $v \geq c_1
  \sqrt{\lg(n/T)/T} \Rightarrow T \geq c_1^2 \lg(n/T)/v^2$. Using the assumption
  $T \leq \lg(nv^2)/v^2$ we get $T \geq c_1^2 \lg(nv^2/\lg(nv^2))/v^2
  = c_1^2 (\lg(nv^2)-\lg\lg(nv^2))/v^2$. But $\lg \lg(nv^2) \leq
  \lg(nv^2)*0.9$ since $nv^2 > 1$. Therefore we conclude $T \geq
  (c_1^2/10) \lg(nv^2)/v^2$.
\item If $T > \lg(nv^2)/v^2$: This case is trivial as we already
  satisfy the claim.
\end{enumerate}

\section{Additional Experiments and Experimental Details}
\label{sec:appendix_experimental_details}

We train a LightGBM classifier and sparsify it in two ways: \Cref{lemma:margin_preservation} and importance sampling. Both sparsification algorithms are random so we repeated both sparsifications 10 times. The experiment was performed on the following dataset: 
\begin{itemize}
\item Inspired by the XGBoost article \cite{xgboost} we used the Higgs dataset \cite{higgs}. We shuffled the $10^7$ training examples and selected the first $10^6$ examples for training and the following $10^6$ points for examples. 
\item Inspired by the LightGBM article \cite{lightgbm} we used the Flight Delay dataset \cite{airline}. The dataset contain delay times and was turned into binary classification by predicting if the flight was delayed or not. All features of the $10^7$ training examples were one-hot encoded with the Pandas \cite{pandas} function \textit{get\_dummies()} which yielded $660$ features. We shuffled the $10^7$ training examples and selected the first $10^6$ examples for training and the following $10^6$ examples for testing. 
\item Inspired by the equivalence margin article \cite{emargin} we used the Letter dataset \cite{letter}. It was turned into a binary classification problem as done in \cite{emargin}. We shuffled the $20000$ examples and used the first $10000$ for training and the last $10000$ for testing. 
\end{itemize}
The initial classifiers were trained with LightGBM using default parameters (except for the Letter dataset where we used LightGBM decision stumps inspired by \cite{emargin}). The rest of this article contain figures that show different variants our experiments. Each figure concerns a single dataset for a choice of \textit{number of hypotheses} $T$ and \textit{number of points} $n$. For example \Cref{fig:airline1} contains results for our experiment on the Flight Delay dataset with $T=100$ hypotheses and $n=250000$ training and test points. It contains a margin plot similar to \Cref{fig:margins} and test AUC/accuracy plot similar to \Cref{fig:scores}. See \Cref{table:experiments} for an overview of all experiments. 

\begin{table}[h]
\centering
\begin{tabular}{|l|l|l|}
\hline
\textbf{Dataset} & \textbf{n} & \textbf{T} \\
\hline
Letter & 10000 & 100 \\
\hline
Flight Delay & 250000 & 100 \\
\hline
Flight Delay & 500000 & 250 \\
\hline
Flight Delay & 1000000 & 500 \\
\hline
Higgs & 250000 & 100 \\
\hline
Higgs & 500000 & 250 \\
\hline
Higgs & 1000000 & 500 \\
\hline
\end{tabular}
\caption{The different experimental settings. }
\label{table:experiments}
\end{table}

Finally, we would like to acknowledge the NumPy, Pandas and Scikit-Learn libraries \cite{numpy, pandas, sklearn}. 

\subsection{Correcting Sparsified Predictions}
\label{sec:correct}
In some cases the classification accuracy of a sparsified classifier is very poor even though the test AUC is good (e.g. accuracy $60\%$ but AUC 0.80). The AUC depends on the relative ordering of predictions while classification accuracy depends on whether each prediction is above or below zero. This lead us to believe that the poor test classification accuracy was caused by a bad offset. Maybe we should predict +1 if points where above $-0.01$ instead of $0$. In other words, it seemed that the sparsified classifiers skewed the bias term of the original classifier. To fix this we computed the bias term that yielded the largest classification accuracy on the training set. This can be done by sorting predictions and then trying every possible offset, one for each point, taking just $O(n\lg(n))$ time. For the airline dataset this typically improved the sparsified classifiers accuracy from $60\%$ to $80\%$. 

In this way we "corrected" the bias of all sparsified predictions. All test classification accuracies reported are corrected in this sense (including \Cref{fig:scores}). 



\begin{figure*}[h]
\centering
\includegraphics[width=0.48\textwidth]{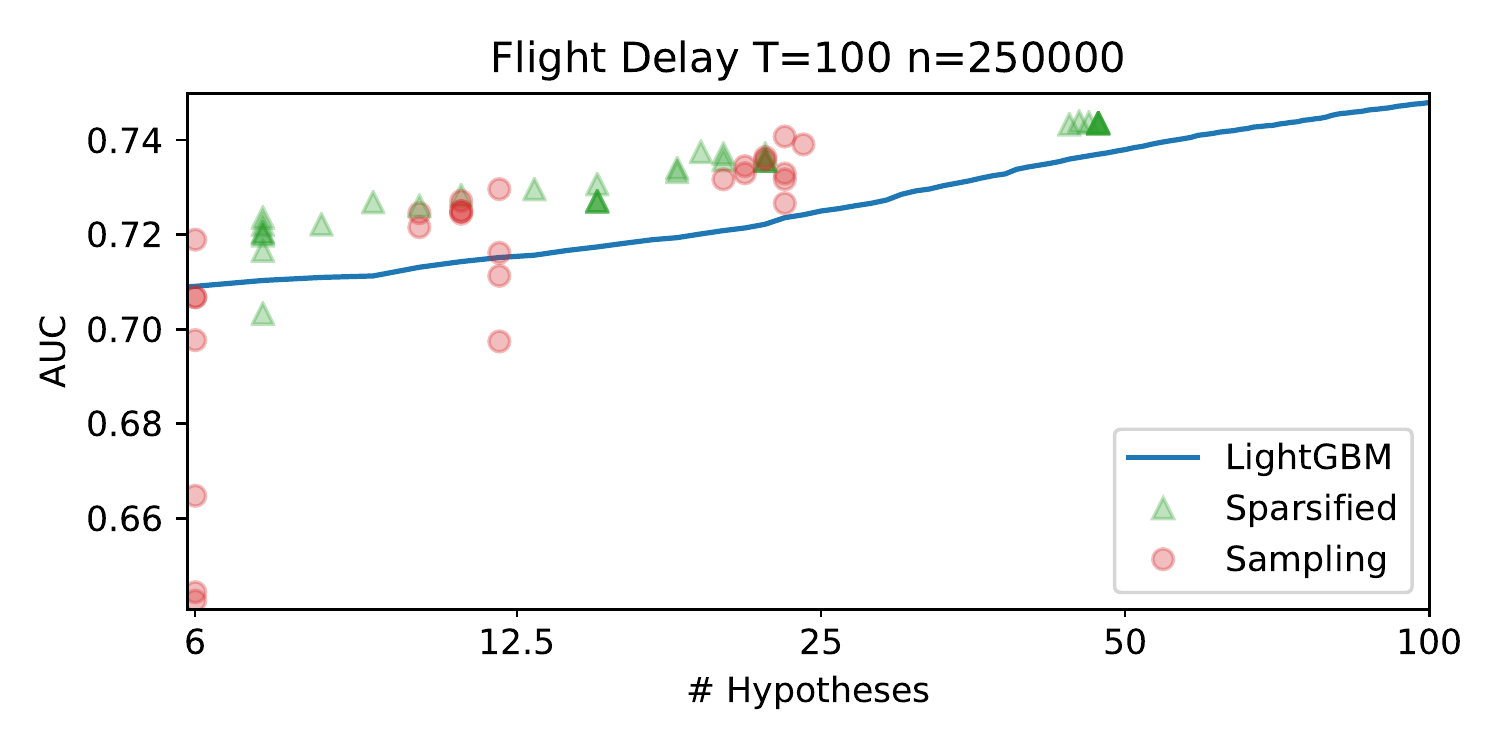}
\includegraphics[width=0.48\textwidth]{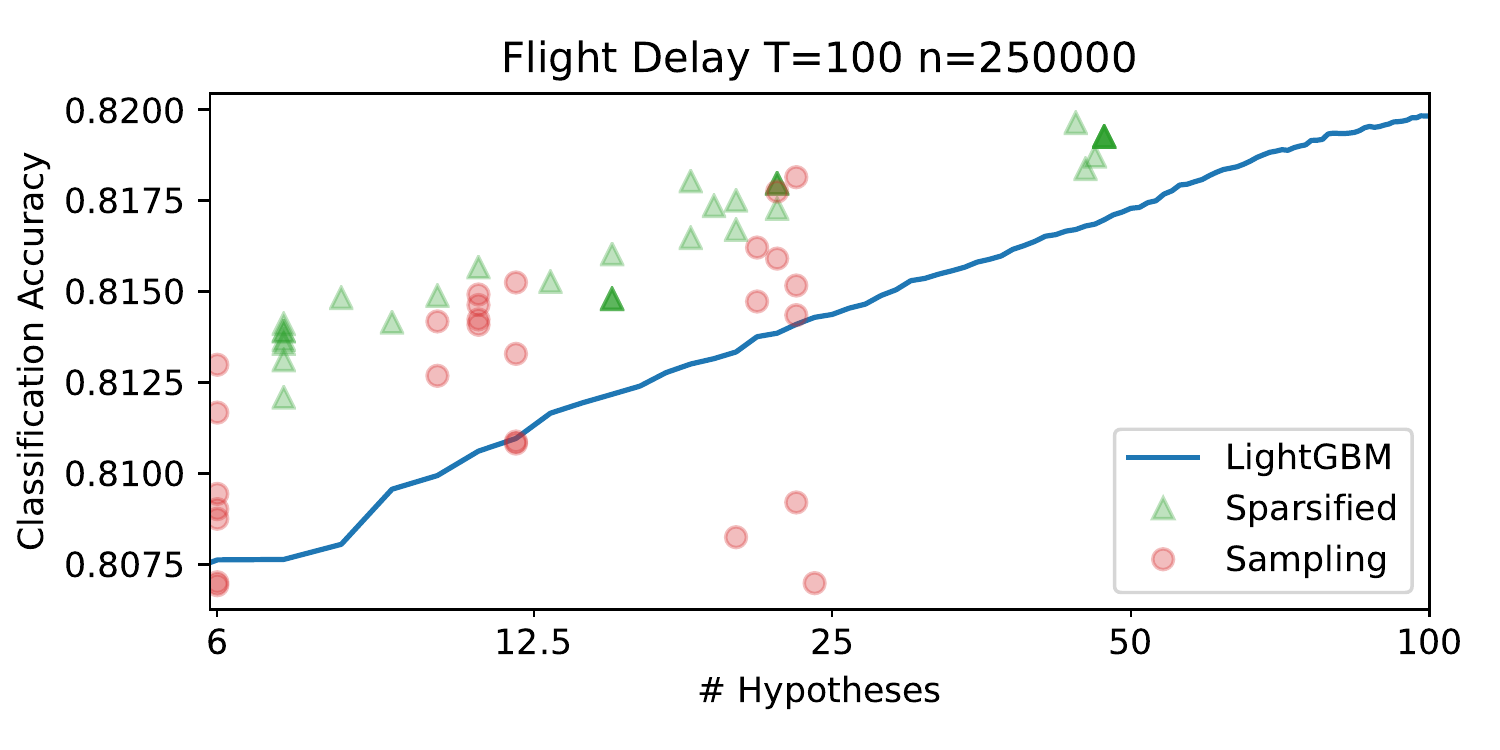}
\includegraphics[width=0.90\textwidth]{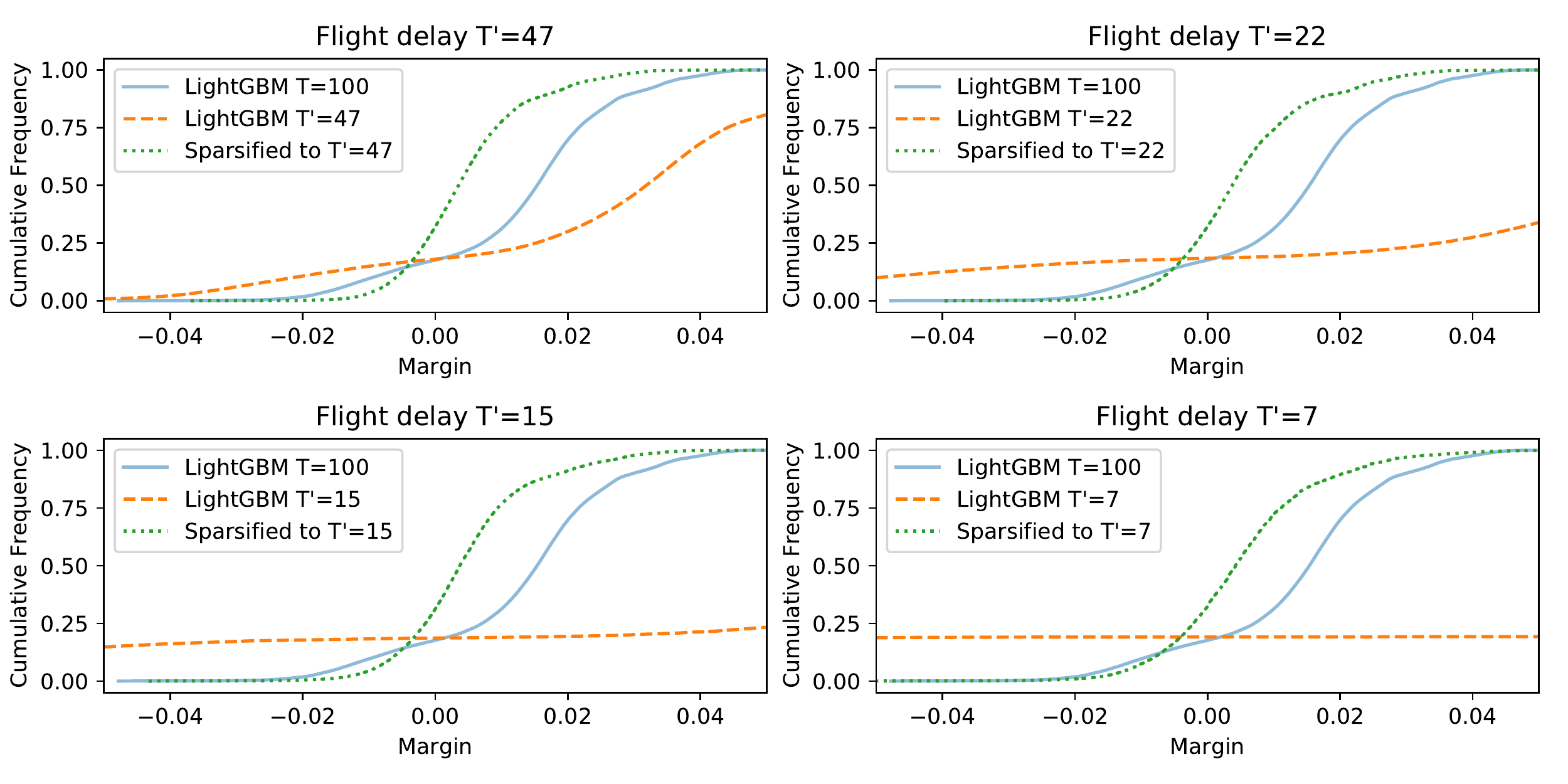}
\caption{Similar to \Cref{fig:scores} and \Cref{fig:margins}, but for Flight Delay with $T=100$ and $250000$ training and test points. }
\label{fig:airline1}
\end{figure*}

\begin{figure*}[h]
\centering
\includegraphics[width=0.48\textwidth]{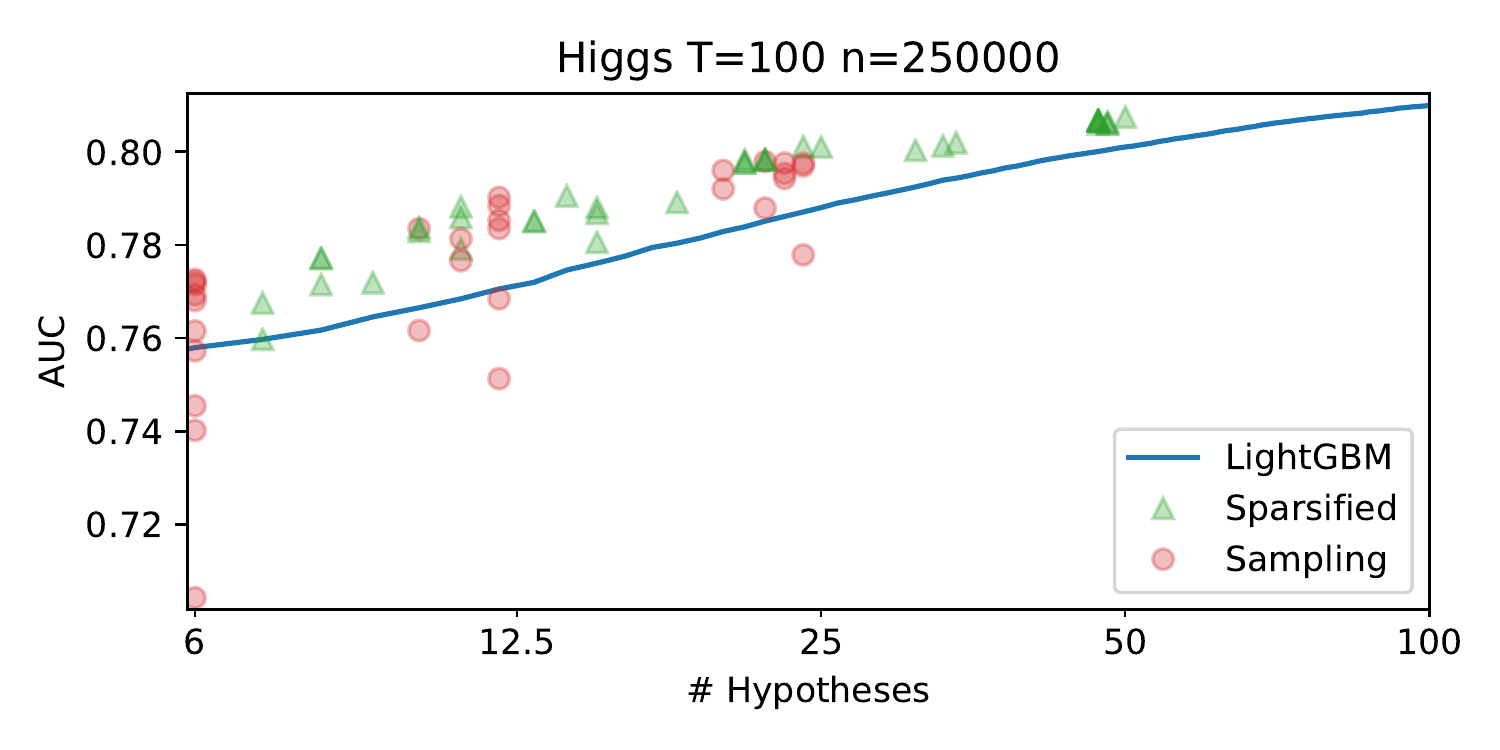}
\includegraphics[width=0.48\textwidth]{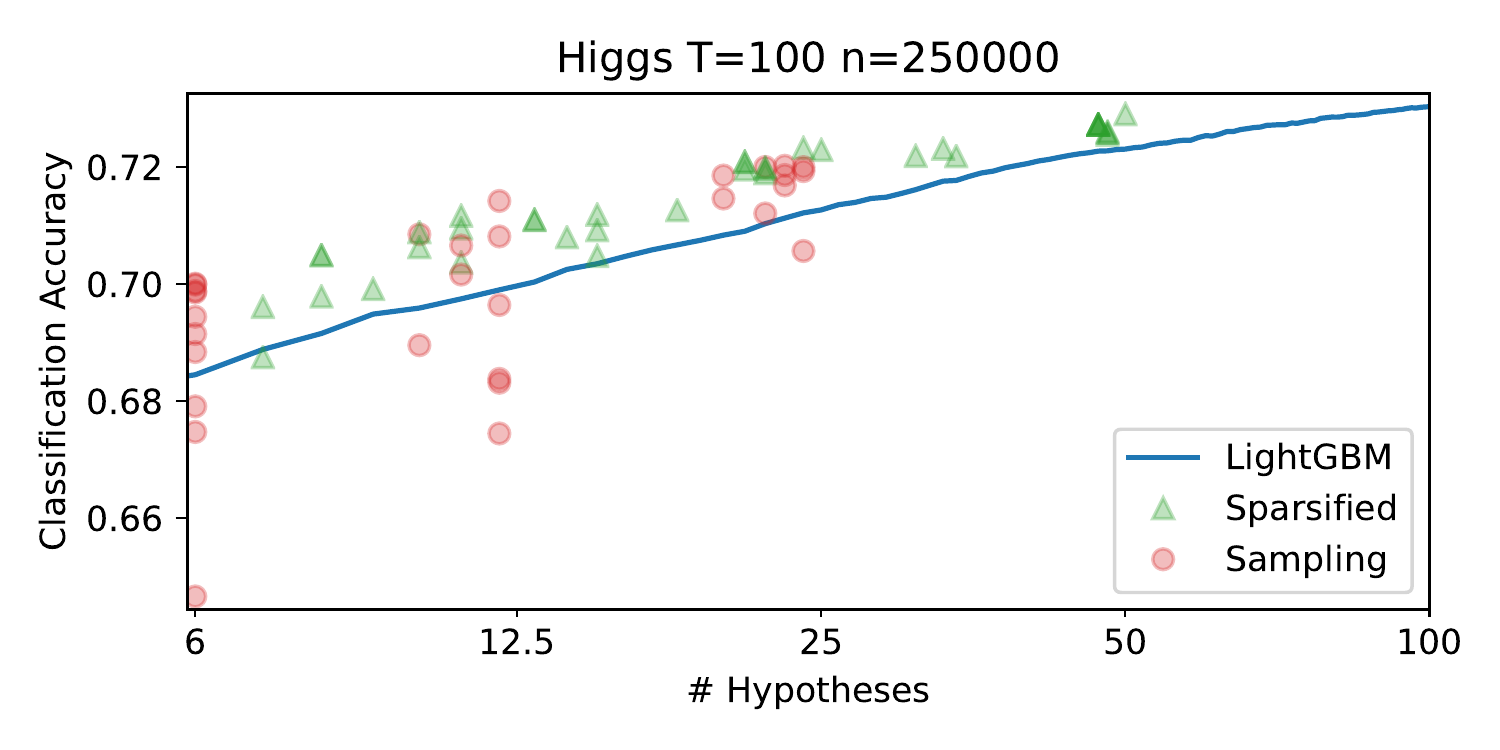}
\hspace{5mm}\includegraphics[width=0.90\textwidth]{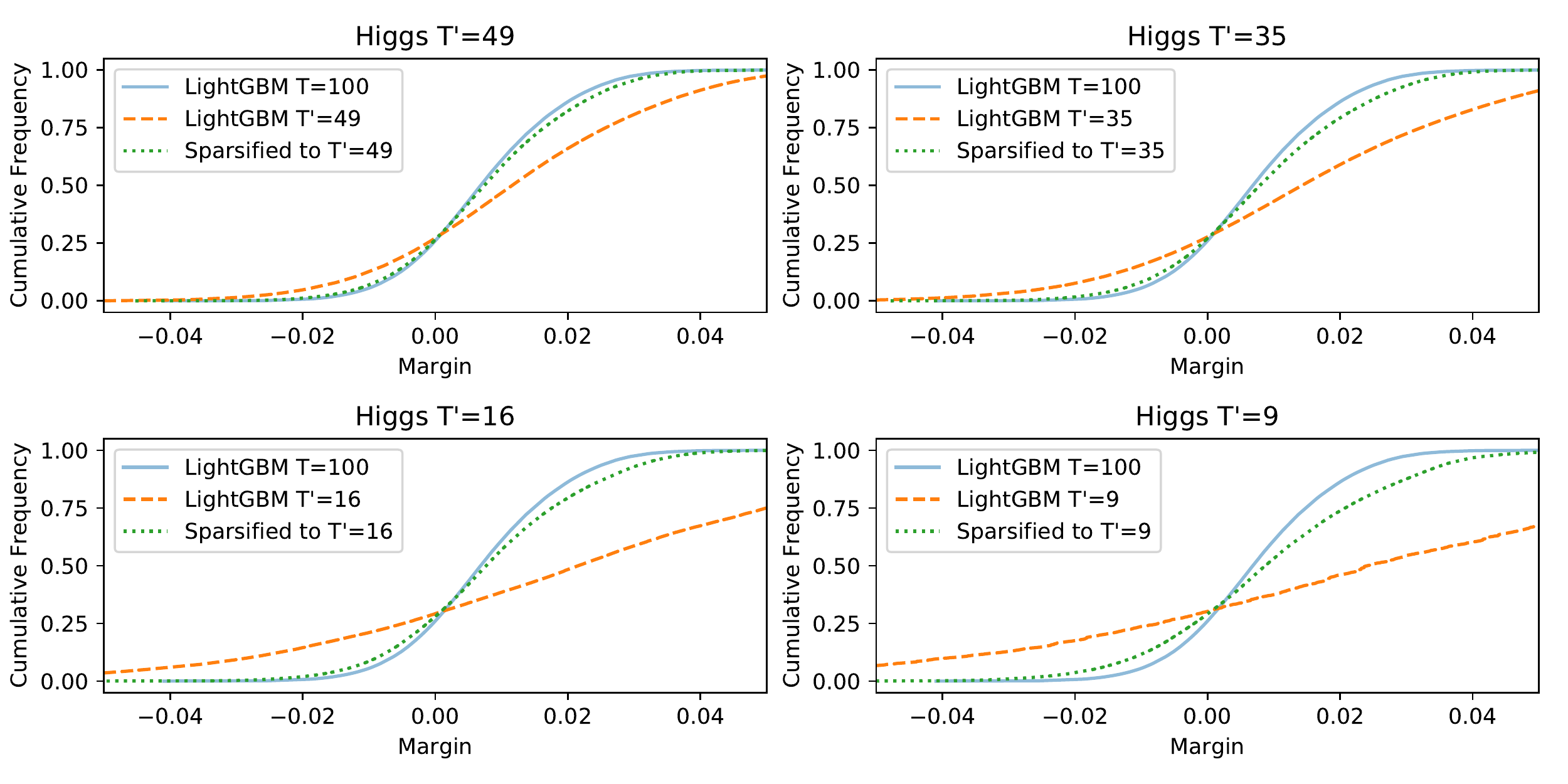}
\caption{Similar to previous plot, but for Higgs with $T=100$ hypotheses and $250000$ training and test points. }
\label{fig:higgs1}
\end{figure*}

\begin{figure*}[h]
\centering
\includegraphics[width=0.48\textwidth]{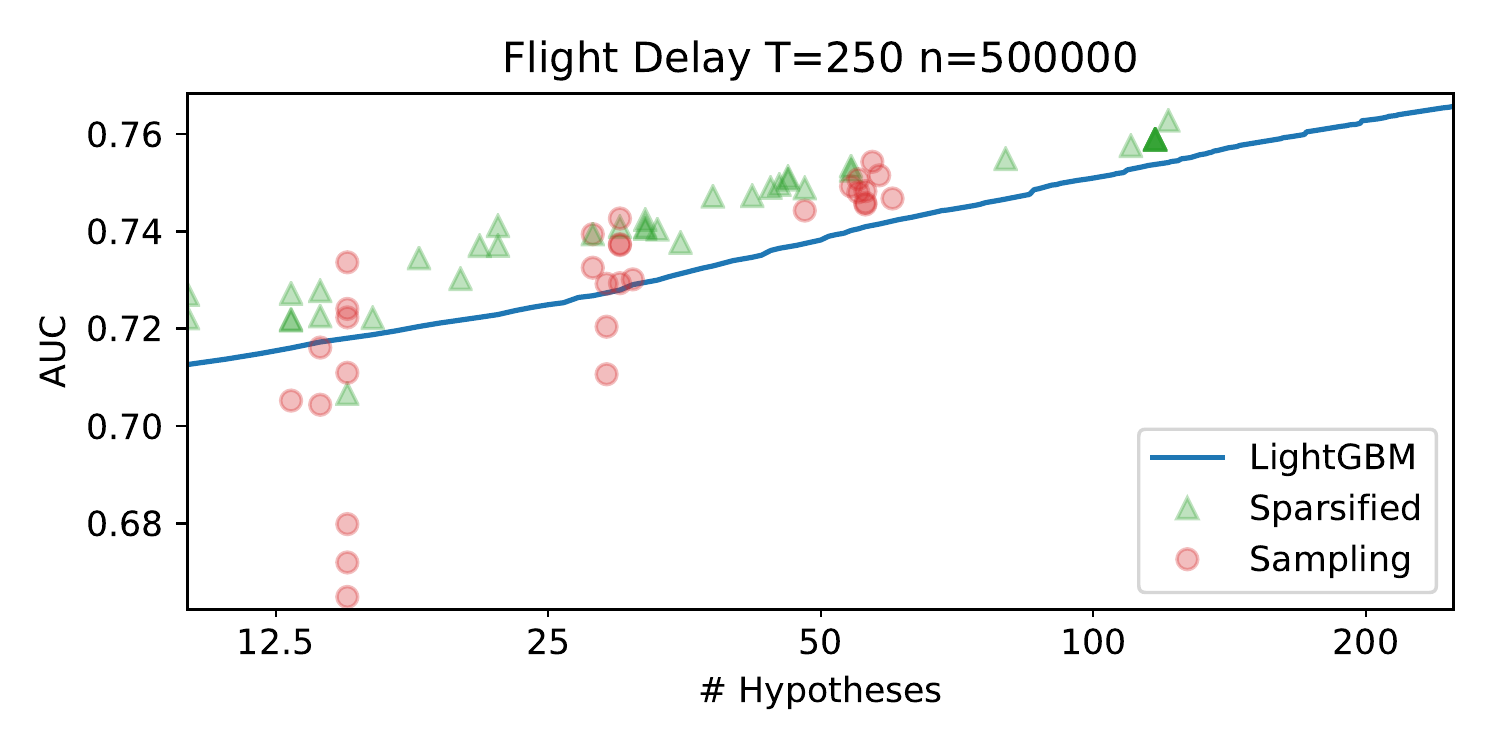}
\includegraphics[width=0.48\textwidth]{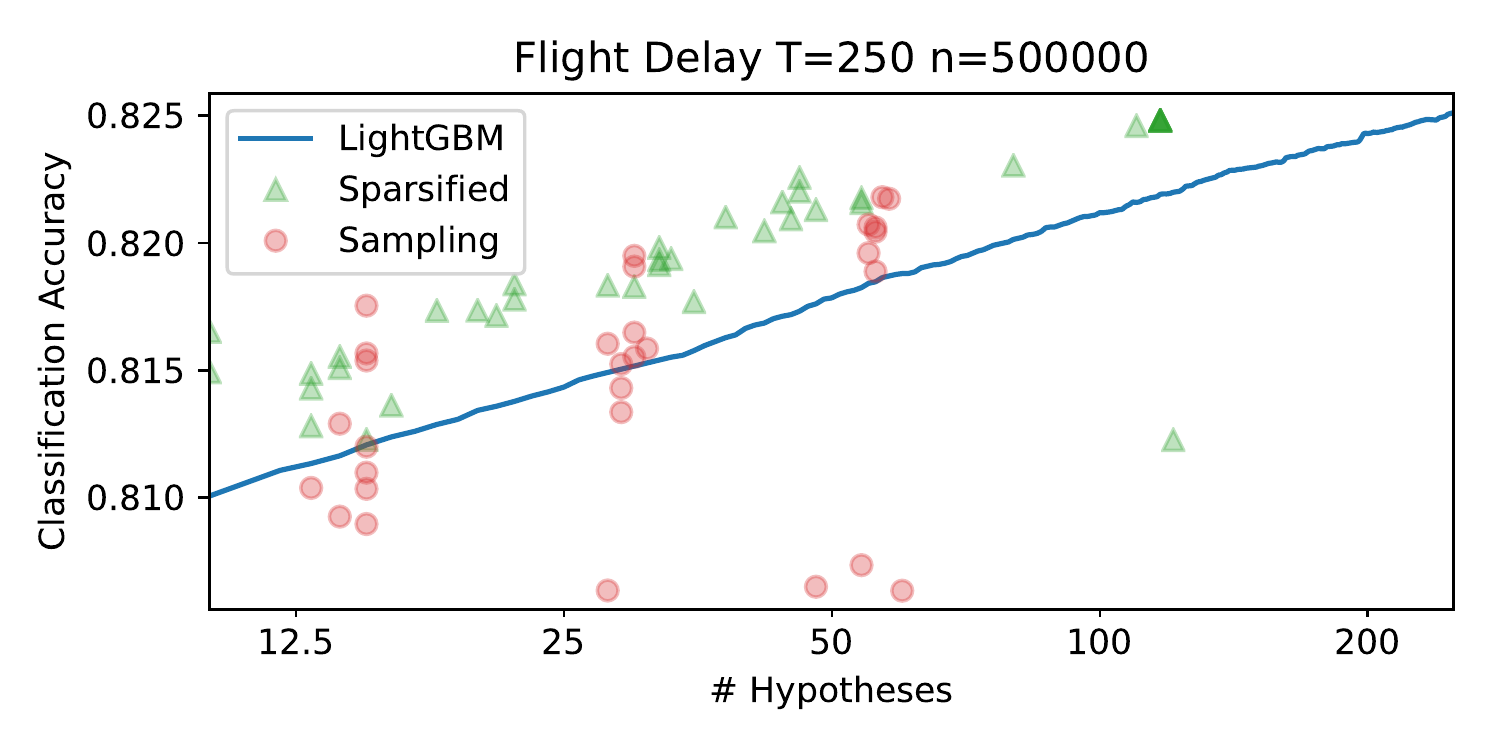}
\includegraphics[width=0.90\textwidth]{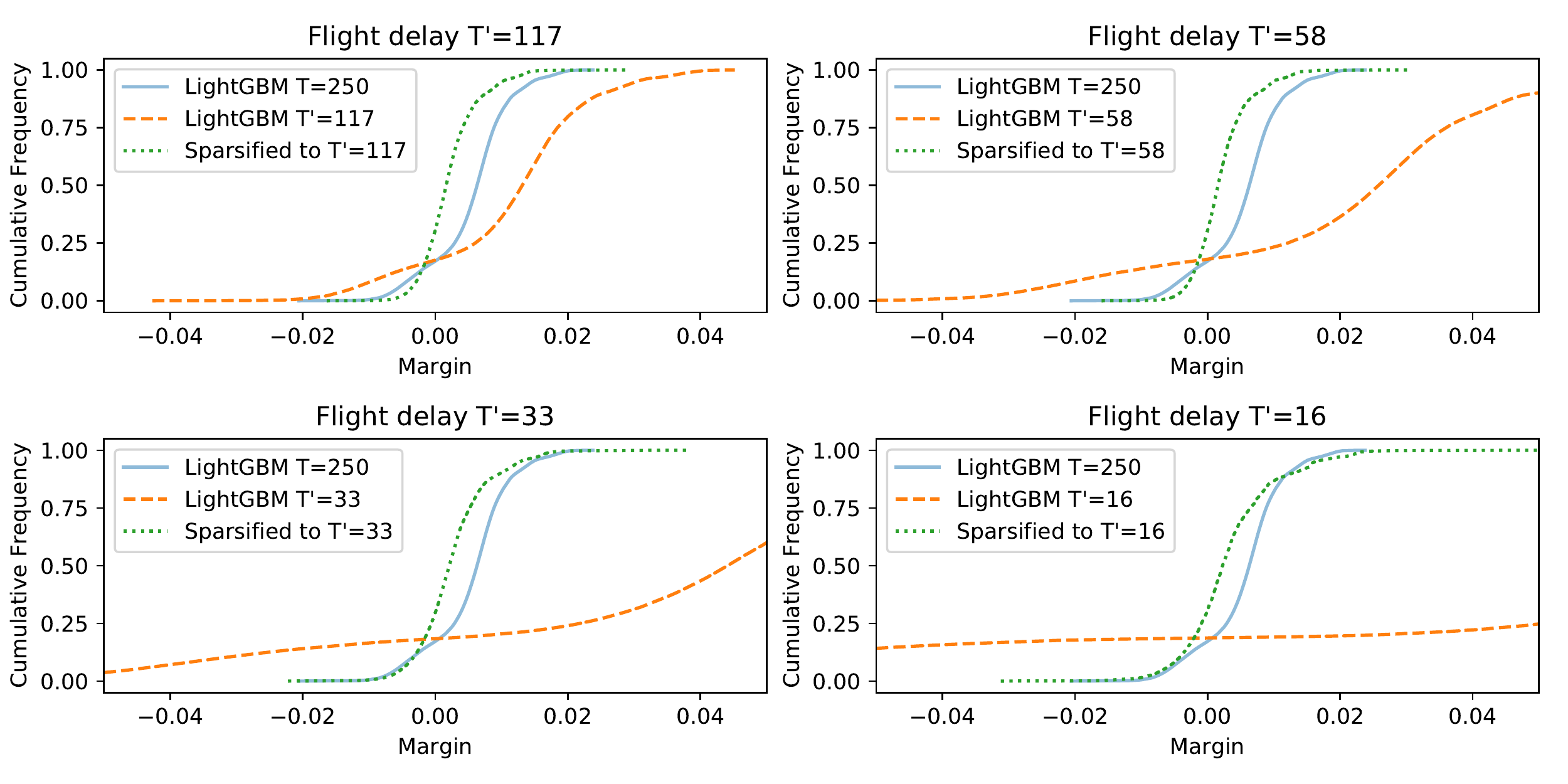}
\caption{Similar to previous plot, but for Flight Delay with $T=250$ hypotheses and $500000$ training and test points. }
\label{fig:airline2}
\end{figure*}

\begin{figure*}[h]
\centering
\includegraphics[width=0.48\textwidth]{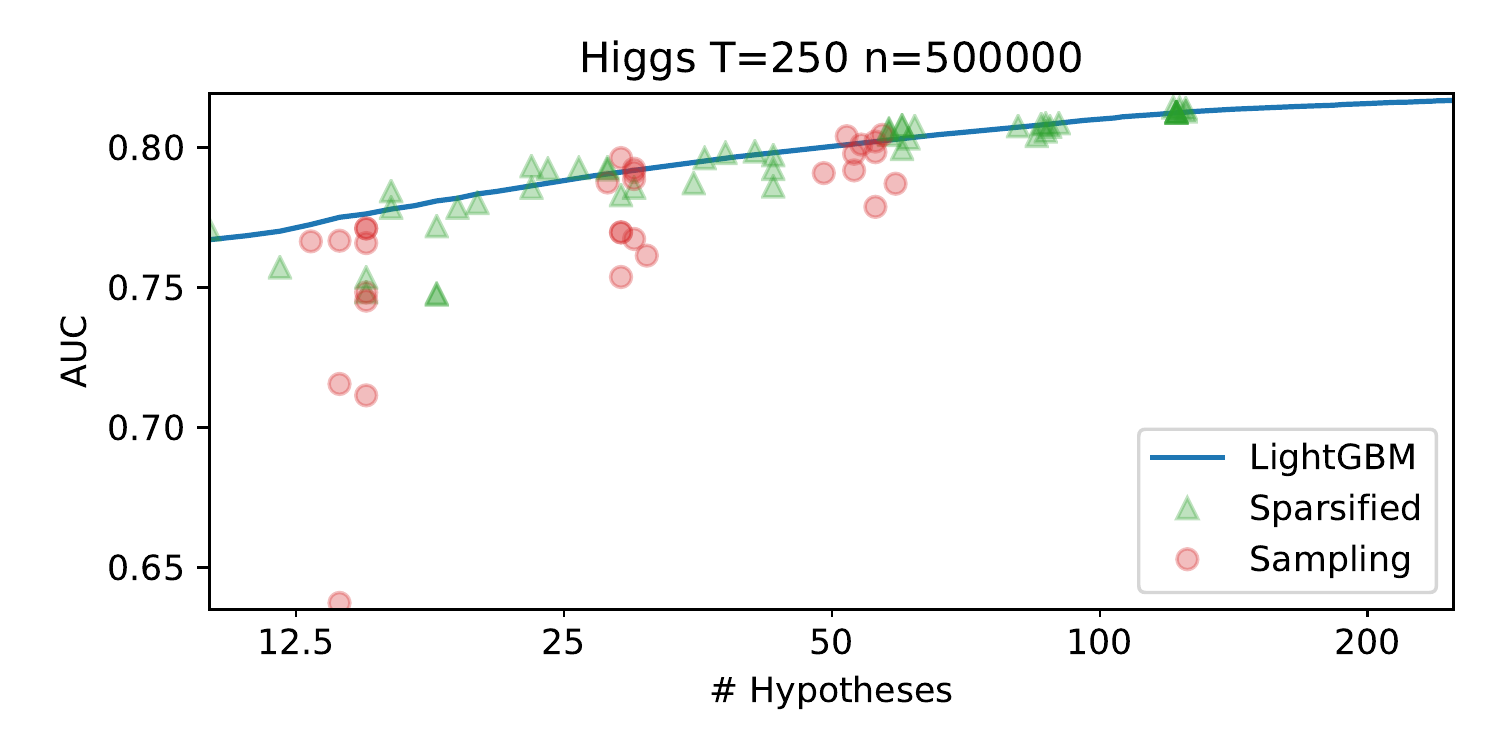}
\includegraphics[width=0.48\textwidth]{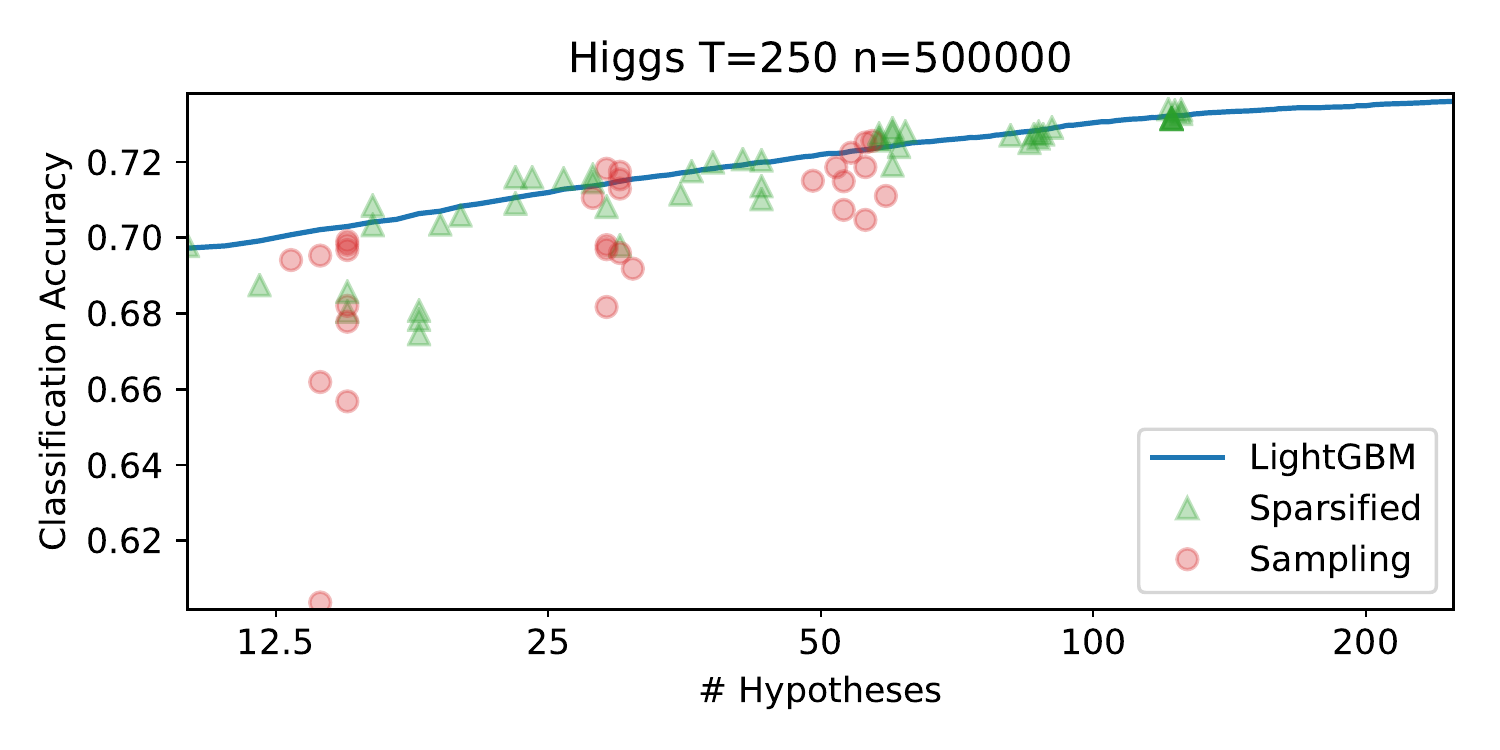}
\hspace{5mm}\includegraphics[width=0.90\textwidth]{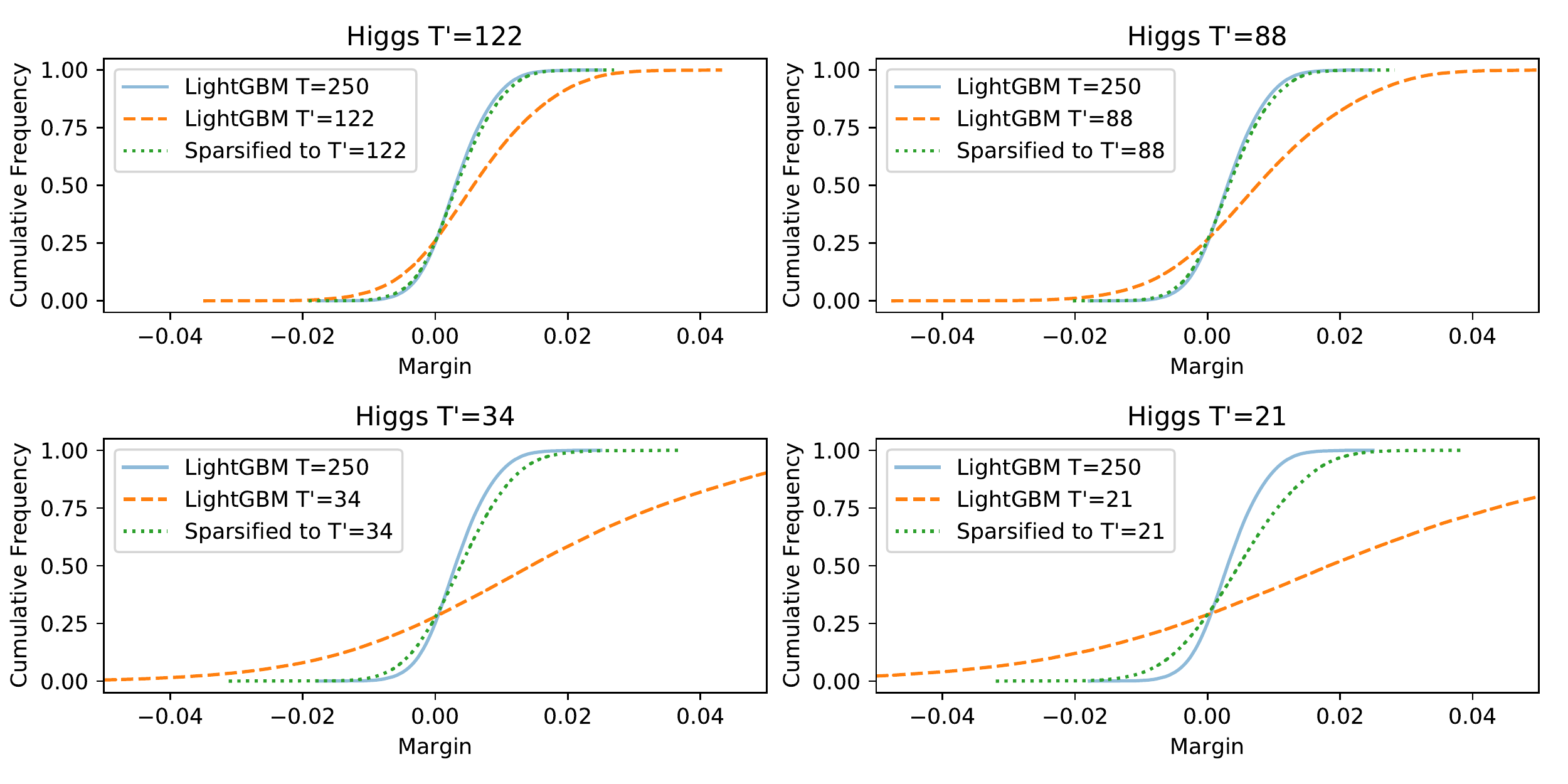}
\caption{Similar to previous plot, but for Higgs with $T=250$ hypotheses and $500000$ training and test points. }
\label{fig:higgs2}
\end{figure*}

\begin{figure*}[h]
\centering
\includegraphics[width=0.48\textwidth]{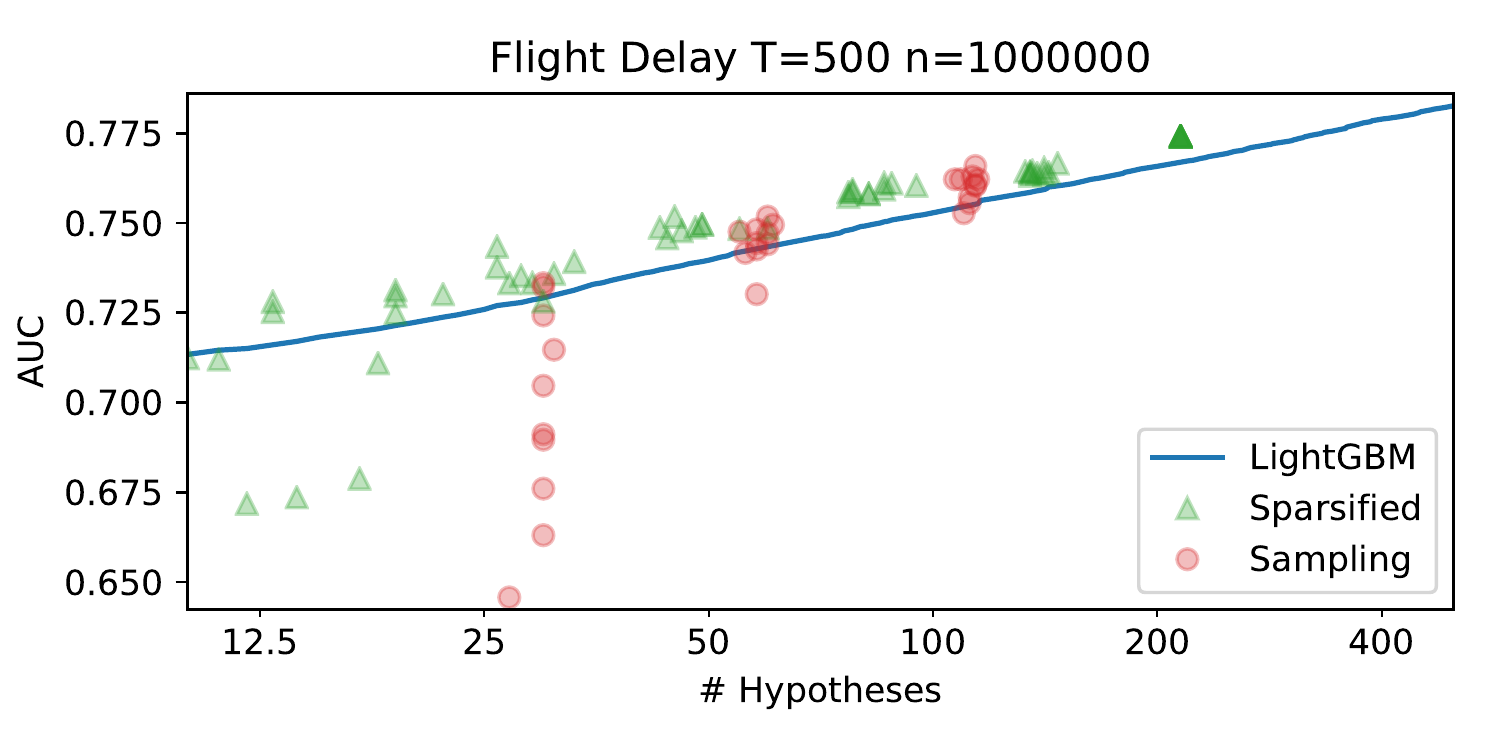}
\includegraphics[width=0.48\textwidth]{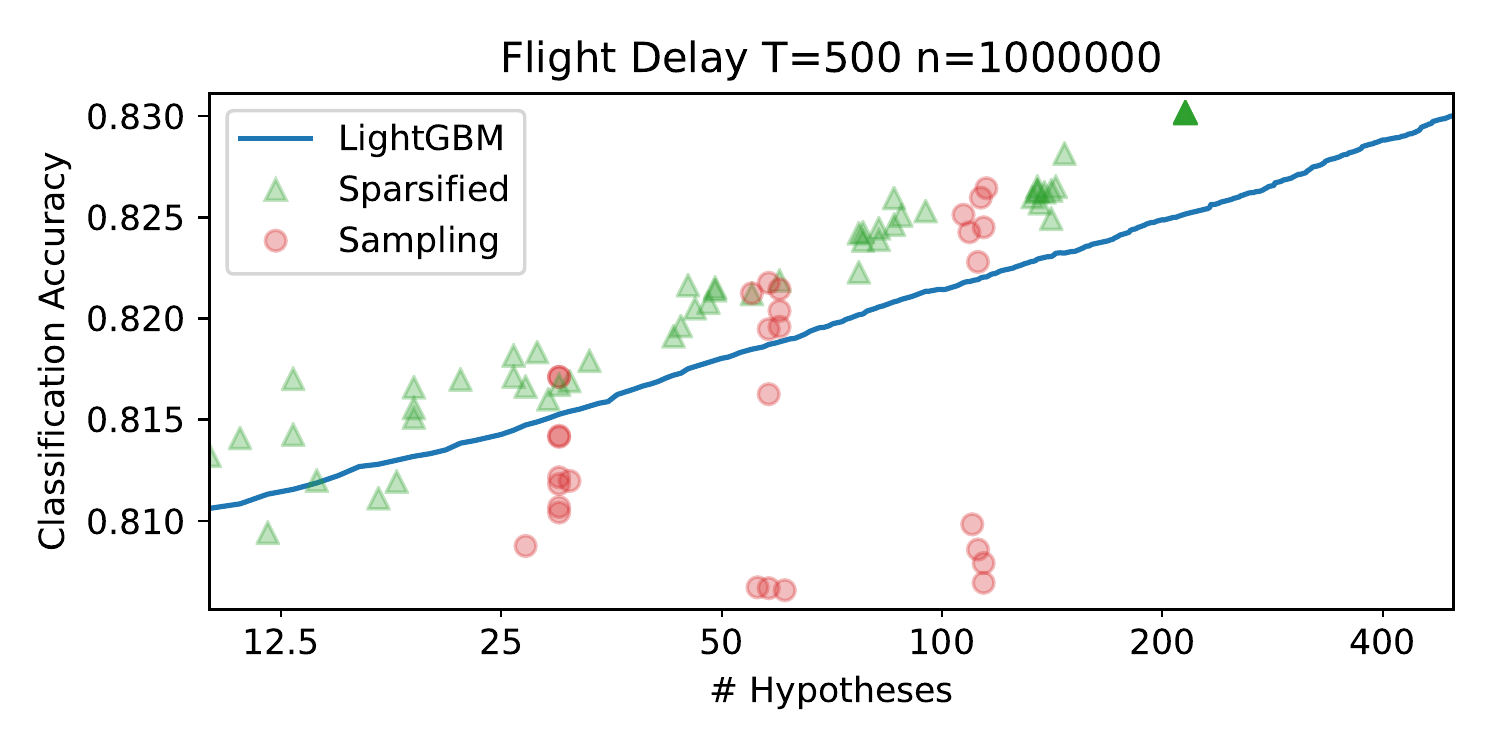}
\includegraphics[width=0.90\textwidth]{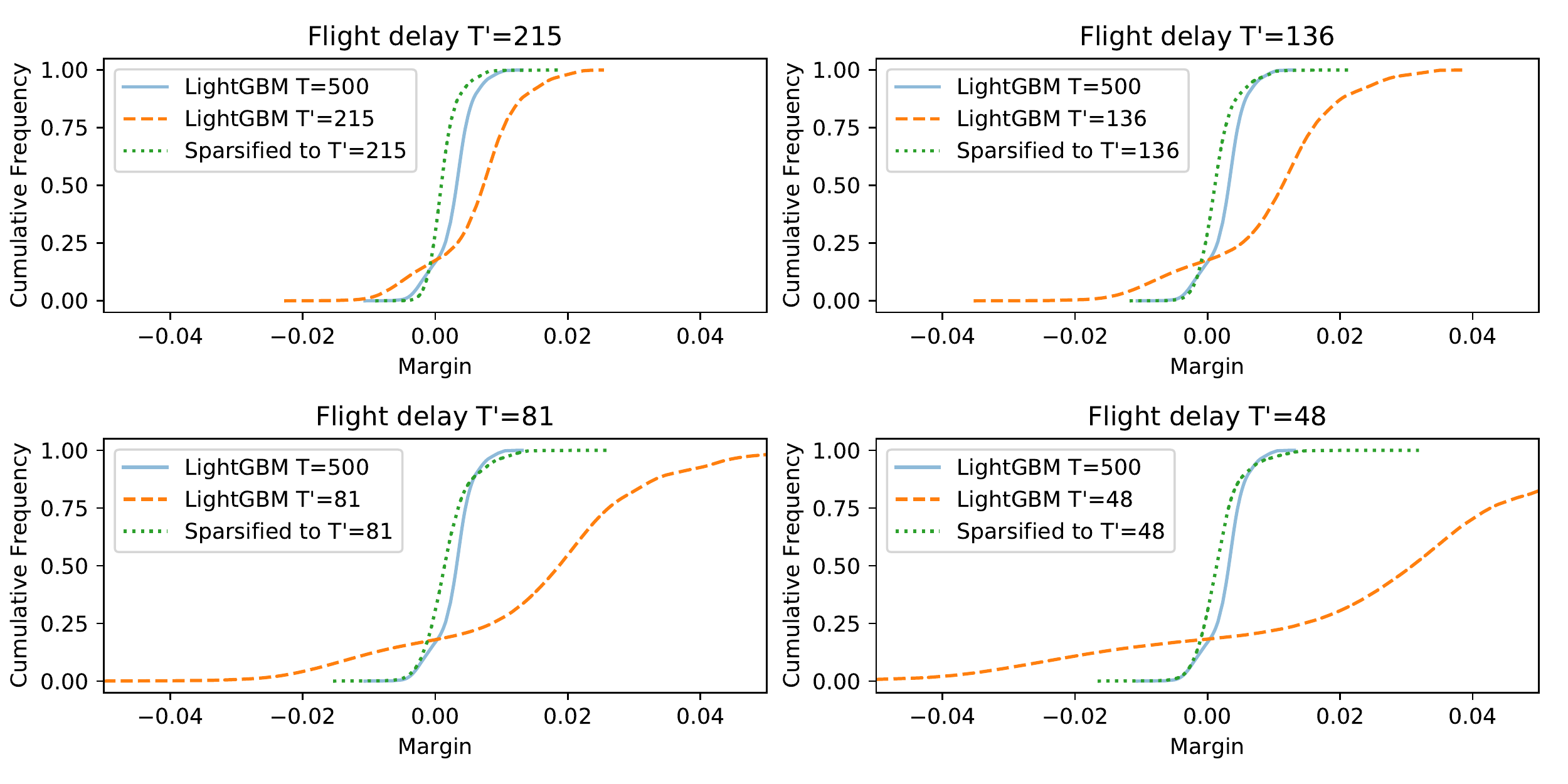}
\caption{Similar to previous plot, but for Flight Delay with $T=500$ hypotheses and $1000000$ training and test points. }
\label{fig:airline3}
\end{figure*}

\begin{figure*}[h]
\centering
\includegraphics[width=0.48\textwidth]{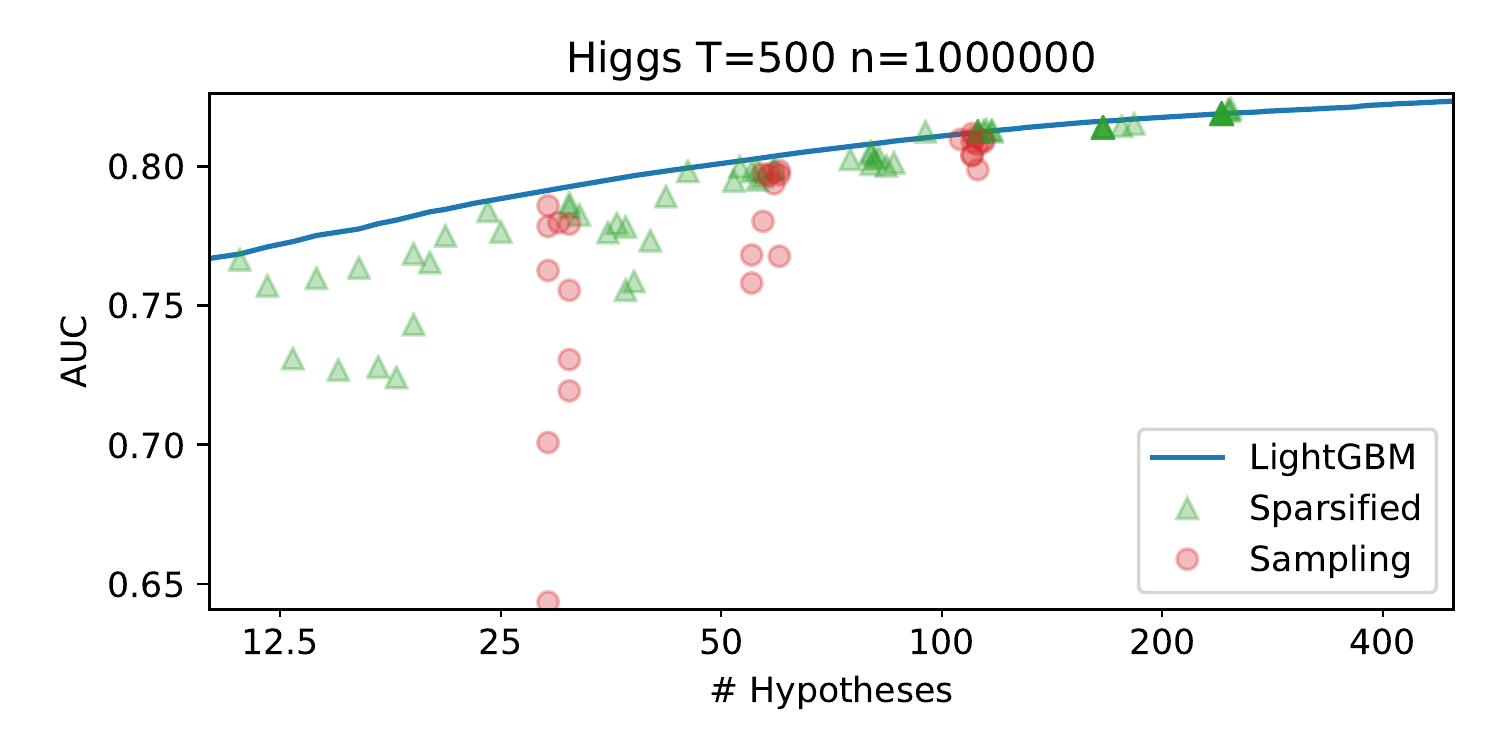}
\includegraphics[width=0.48\textwidth]{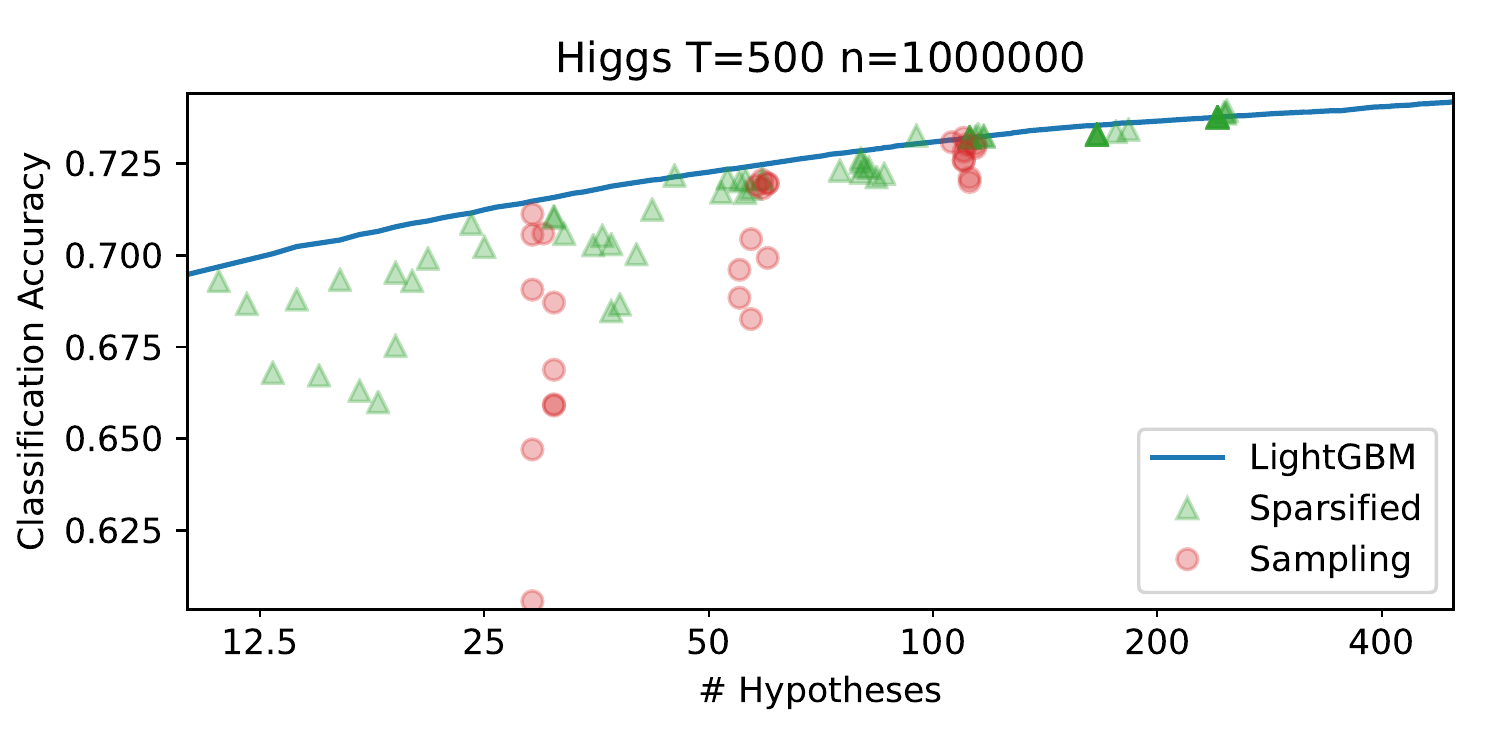}
\hspace{5mm}\includegraphics[width=0.90\textwidth]{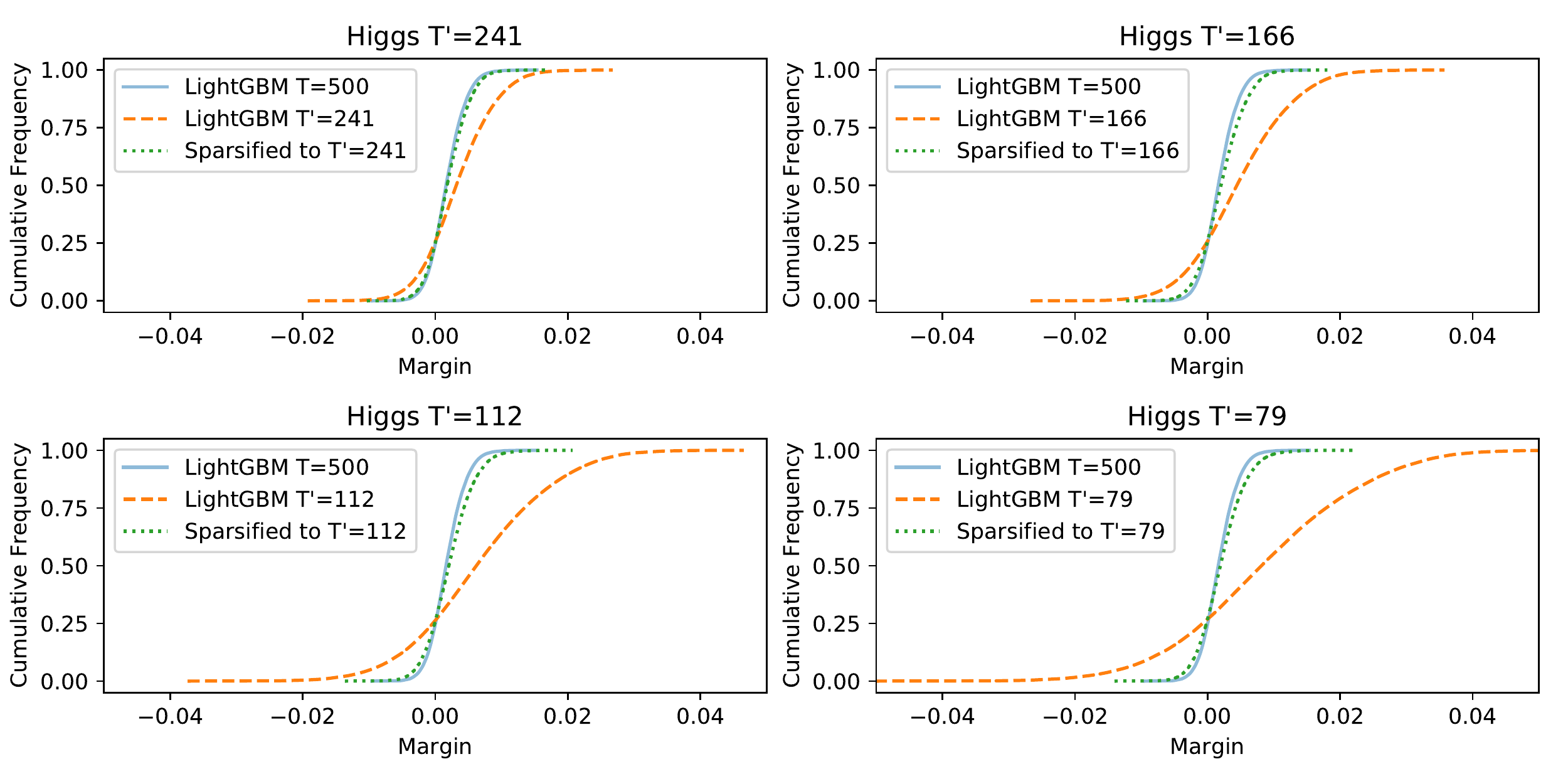}
\caption{Similar to previous plot, but for Higgs with $T=500$ hypotheses and $1000000$ training and test points. }
\label{fig:higgs3}
\end{figure*}

\begin{figure*}[h]
\centering
\includegraphics[width=0.48\textwidth]{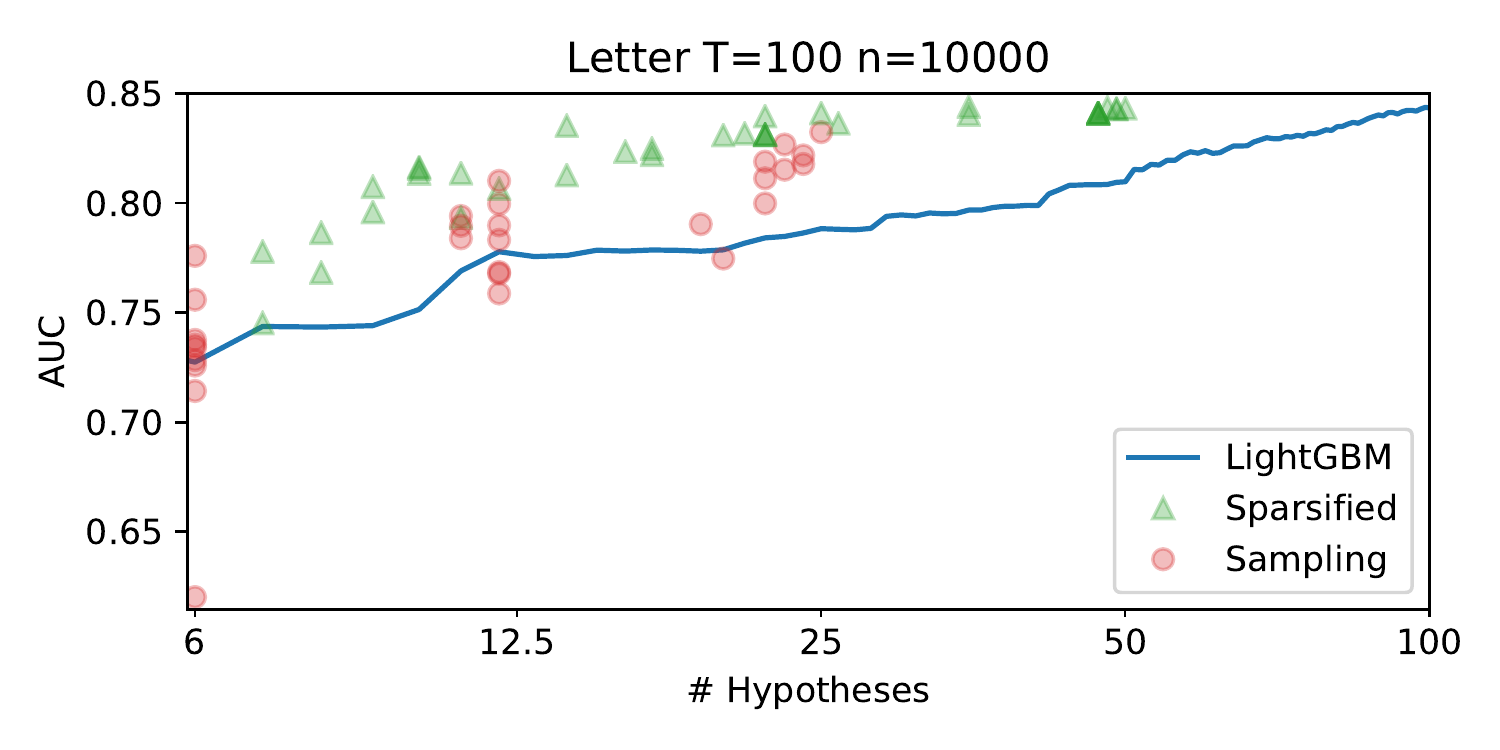}
\includegraphics[width=0.48\textwidth]{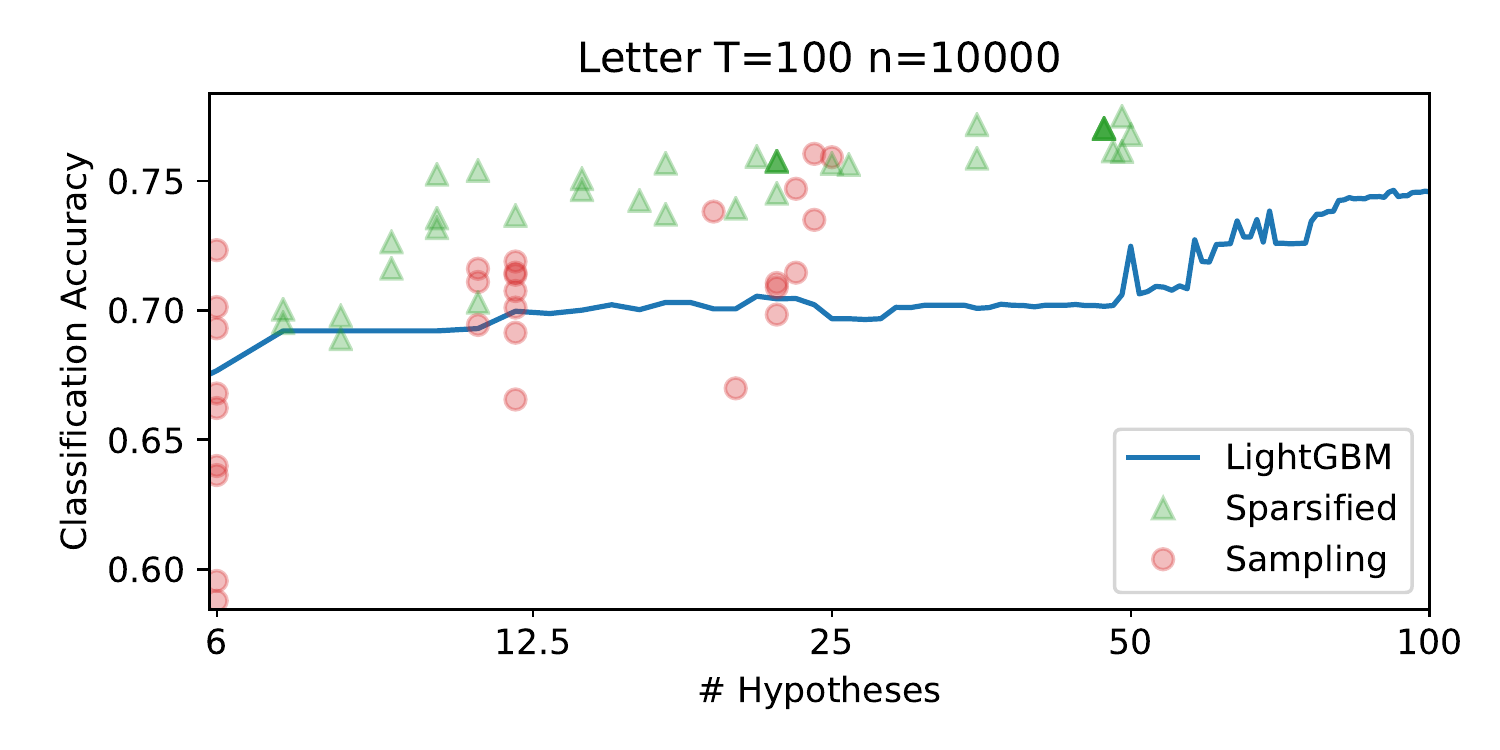}
\includegraphics[width=0.90\textwidth]{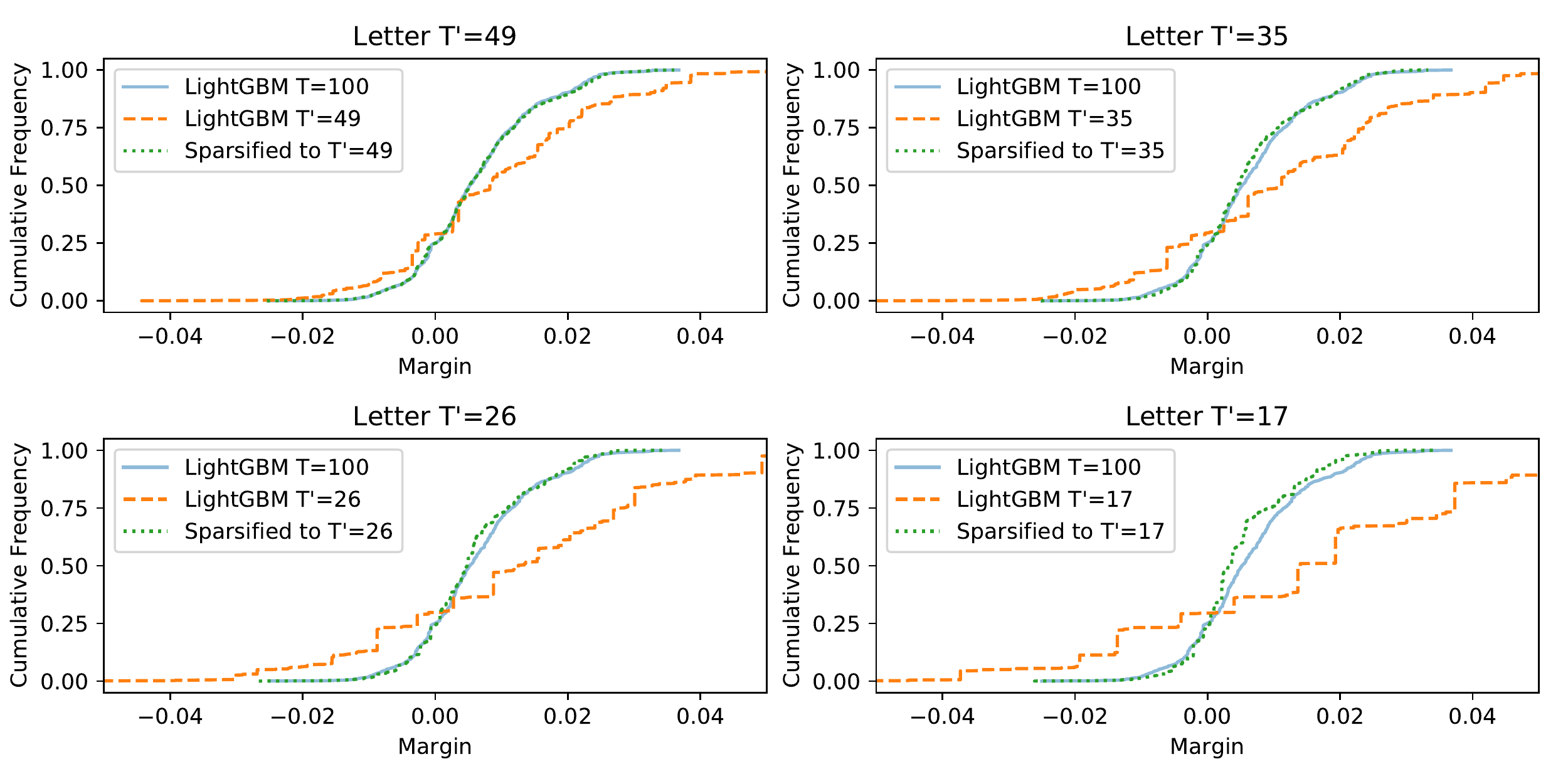}
\caption{Similar to previous plot, but for Letter with $T=100$ and $n=10000$ training and test points. Furthermore, LightGBM used decision stumps (inspired by \cite{emargin}), that is, we choose $\text{num\_leaves}=2$. Quite surprisingly the sparsified classifiers obtain a better classification accuracy than the original classifier. All sparsified classifiers predictions were corrected as described in \Cref{sec:correct}. Since the original classifier has better AUC than the sparsified ones, we believe this is caused by a poor bias for the original classifier. }
\label{fig:letter}
\end{figure*}

%% file: arxiv_full_version.bbl
\begin{thebibliography}{}

\bibitem[air, ]{airline}
Flight delay data.
\newblock \url{https://github.com/szilard/benchm-ml#data}.

\bibitem[Bennett et~al., 2000]{bennett2000column}
Bennett, K.~P., Demiriz, A., and Shawe-Taylor, J. (2000).
\newblock A column generation algorithm for boosting.
\newblock In {\em ICML}, pages 65--72.

\bibitem[Breiman, 1999]{breiman1999prediction}
Breiman, L. (1999).
\newblock Prediction games and arcing algorithms.
\newblock {\em Neural computation}, 11(7):1493--1517.

\bibitem[Chen and Guestrin, 2016]{xgboost}
Chen, T. and Guestrin, C. (2016).
\newblock Xgboost: A scalable tree boosting system.
\newblock In {\em Proceedings of the 22nd acm sigkdd international conference
  on knowledge discovery and data mining}, pages 785--794. ACM.

\bibitem[Dheeru and Karra~Taniskidou, 2017]{letter}
Dheeru, D. and Karra~Taniskidou, E. (2017).
\newblock {UCI} machine learning repository.
\newblock \url{http://archive.ics.uci.edu/ml}.

\bibitem[Freund et~al., 1999]{freund1999short}
Freund, Y., Schapire, R., and Abe, N. (1999).
\newblock A short introduction to boosting.
\newblock {\em Journal-Japanese Society For Artificial Intelligence},
  14(771-780):1612.

\bibitem[Freund and Schapire, 1995]{freund1995desicion}
Freund, Y. and Schapire, R.~E. (1995).
\newblock A desicion-theoretic generalization of on-line learning and an
  application to boosting.
\newblock In {\em European conference on computational learning theory}, pages
  23--37. Springer.

\bibitem[Friedman, 2001]{gb2}
Friedman, J.~H. (2001).
\newblock Greedy function approximation: a gradient boosting machine.
\newblock {\em Annals of statistics}, pages 1189--1232.

\bibitem[Gao and Zhou, 2013]{kmargin}
Gao, W. and Zhou, Z.-H. (2013).
\newblock On the doubt about margin explanation of boosting.
\newblock {\em Artificial Intelligence}, 203:1--18.

\bibitem[Grove and Schuurmans, 1998]{grove1998boosting}
Grove, A.~J. and Schuurmans, D. (1998).
\newblock Boosting in the limit: Maximizing the margin of learned ensembles.
\newblock In {\em AAAI/IAAI}, pages 692--699.

\bibitem[Ke et~al., 2017]{lightgbm}
Ke, G., Meng, Q., Finley, T., Wang, T., Chen, W., Ma, W., Ye, Q., and Liu,
  T.-Y. (2017).
\newblock Lightgbm: A highly efficient gradient boosting decision tree.
\newblock In {\em Advances in Neural Information Processing Systems}, pages
  3146--3154.

\bibitem[Klein and Young, 1999]{klein1999number}
Klein, P. and Young, N. (1999).
\newblock On the number of iterations for dantzig-wolfe optimization and
  packing-covering approximation algorithms.
\newblock {\em Lecture Notes in Computer Science}, 1610:320--327.

\bibitem[Koltchinskii et~al., 2001]{koltchinskii2001some}
Koltchinskii, V., Panchenko, D., and Lozano, F. (2001).
\newblock Some new bounds on the generalization error of combined classifiers.
\newblock In {\em Advances in neural information processing systems}, pages
  245--251.

\bibitem[Larsen, 2017]{discmin2}
Larsen, K.~G. (2017).
\newblock Constructive discrepancy minimization with hereditary {L2}
  guarantees.
\newblock {\em CoRR}, abs/1711.02860.

\bibitem[Lovett and Meka, 2015]{discmin1}
Lovett, S. and Meka, R. (2015).
\newblock Constructive discrepancy minimization by walking on the edges.
\newblock {\em SIAM Journal on Computing}, 44(5):1573--1582.

\bibitem[Mason et~al., 2000]{gb1}
Mason, L., Baxter, J., Bartlett, P.~L., and Frean, M.~R. (2000).
\newblock Boosting algorithms as gradient descent.
\newblock In Solla, S.~A., Leen, T.~K., and M\"{u}ller, K., editors, {\em
  Advances in Neural Information Processing Systems 12}, pages 512--518. MIT
  Press.

\bibitem[McKinney et~al., 2010]{pandas}
McKinney, W. et~al. (2010).
\newblock Data structures for statistical computing in python.
\newblock In {\em Proceedings of the 9th Python in Science Conference}, volume
  445, pages 51--56. Austin, TX.

\bibitem[Montgomery-Smith, 1990]{rademacherbounds}
Montgomery-Smith, S. (1990).
\newblock The distribution of {R}ademacher sums.
\newblock {\em Proc. Amer. Math. Soc.}, 109:517--522.

\bibitem[Nie et~al., 2013]{nie2013open}
Nie, J., Warmuth, M., Vishwanathan, S., and Zhang, X. (2013).
\newblock Open problem: Lower bounds for boosting with hadamard matrices.
\newblock {\em Journal of Machine Learning Research}, 30:1076--1079.

\bibitem[Oliphant, 2006]{numpy}
Oliphant, T.~E. (2006).
\newblock {\em A guide to NumPy}, volume~1.
\newblock Trelgol Publishing USA.

\bibitem[Pedregosa et~al., 2011]{sklearn}
Pedregosa, F., Varoquaux, G., Gramfort, A., Michel, V., Thirion, B., Grisel,
  O., Blondel, M., Prettenhofer, P., Weiss, R., Dubourg, V., Vanderplas, J.,
  Passos, A., Cournapeau, D., Brucher, M., Perrot, M., and Duchesnay, E.
  (2011).
\newblock Scikit-learn: Machine learning in {P}ython.
\newblock {\em Journal of Machine Learning Research}, 12:2825--2830.

\bibitem[Prokhorenkova et~al., 2018]{catboost}
Prokhorenkova, L., Gusev, G., Vorobev, A., Dorogush, A.~V., and Gulin, A.
  (2018).
\newblock Catboost: unbiased boosting with categorical features.
\newblock In {\em Advances in Neural Information Processing Systems}, pages
  6637--6647.

\bibitem[R{\"a}tsch and Warmuth, 2002]{ratsch2002maximizing}
R{\"a}tsch, G. and Warmuth, M.~K. (2002).
\newblock Maximizing the margin with boosting.
\newblock In {\em COLT}, volume 2375, pages 334--350. Springer.

\bibitem[R{\"a}tsch and Warmuth, 2005]{ratsch2005efficient}
R{\"a}tsch, G. and Warmuth, M.~K. (2005).
\newblock Efficient margin maximizing with boosting.
\newblock {\em Journal of Machine Learning Research}, 6(Dec):2131--2152.

\bibitem[Reyzin and Schapire, 2006]{reyzin2006boosting}
Reyzin, L. and Schapire, R.~E. (2006).
\newblock How boosting the margin can also boost classifier complexity.
\newblock In {\em Proceedings of the 23rd international conference on Machine
  learning}, pages 753--760. ACM.

\bibitem[Schapire et~al., 1998]{schapire1998boosting}
Schapire, R.~E., Freund, Y., Bartlett, P., Lee, W.~S., et~al. (1998).
\newblock Boosting the margin: A new explanation for the effectiveness of
  voting methods.
\newblock {\em The annals of statistics}, 26(5):1651--1686.

\bibitem[Spencer, 1985]{spencer1985six}
Spencer, J. (1985).
\newblock Six standard deviations suffice.
\newblock {\em Transactions of the American mathematical society},
  289(2):679--706.

\bibitem[Wang et~al., 2008]{emargin}
Wang, L., Sugiyama, M., Yang, C., Zhou, Z.-H., and Feng, J. (2008).
\newblock On the margin explanation of boosting algorithms.
\newblock In {\em COLT}, pages 479--490. Citeseer.

\bibitem[Whiteson, 2014]{higgs}
Whiteson, D. (2014).
\newblock Higgs data set.
\newblock \url{https://archive.ics.uci.edu/ml/datasets/HIGGS}.

\end{thebibliography}
